\documentclass[sn-mathphys-num]{sn-jnl}

\usepackage{graphicx}%
\usepackage{subcaption}
\usepackage{multirow}%
\usepackage{amsmath,amssymb,amsfonts}%
\usepackage{pifont} 
\newcommand{\cmark}{\ding{51}}%
\newcommand{\xmark}{\ding{55}}%
\usepackage{stmaryrd}
\usepackage{amsthm}%
\usepackage[scr=boondox]{mathalpha}   
\usepackage[title]{appendix}%
\usepackage{xcolor}%
\usepackage{textcomp}%
\usepackage{manyfoot}%
\usepackage{booktabs}%
\usepackage{algorithm}%
\usepackage{algorithmicx}%
\usepackage{algpseudocode}%
\usepackage{listings}%

\usepackage{amsthm}

\theoremstyle{plain}%
\newtheorem{definition}{Definition}%

\theoremstyle{remark}%
\newtheorem{example}{Example}%
\newtheorem{remark}{Remark}%

\theoremstyle{definition}%
\newtheorem{theorem}{Theorem}
\newtheorem{proposition}[theorem]{Proposition}%
\newtheorem{corollary}[theorem]{Corollary}%
\newtheorem*{proposition*}{Proposition} 

\usepackage{stmaryrd}

\usepackage{tikz}
\usetikzlibrary{math}
\usetikzlibrary{calc, arrows.meta, intersections, patterns, positioning, shapes.misc, fadings, through,decorations.pathreplacing}
\tikzstyle{edge} = [fill,opacity=.2,fill opacity=.9,line cap=round, line join=round, line width=50pt]
\usetikzlibrary{fit}

\newcommand{\gridgraph}[1]{
    \begin{tikzpicture}
    \tikzmath{\k = #1-1;\m = #1-2;}
      \foreach \x in {0,1,...,#1} {
        \foreach \y in {0,1,...,#1} {
          \node[circle, fill, inner sep=1pt] at (\x, \y) {};
        }
      }
    
      \foreach \x in {0,1,...,\k} {
        \draw[thick] (\x, #1) -- (\x+1, #1);
        \draw[thick] (#1, \x) -- (#1, \x+1);
        \foreach \y in {0,1,...,\k} {
          \draw[thick] (\x, \y) -- (\x+1, \y);
          \draw[thick] (\x, \y) -- (\x, \y+1);
        }
      }
\end{tikzpicture}
}

\newcommand{\diamondgraph}[1]{
    \begin{tikzpicture}
    \tikzmath{\k = #1-1;\m = #1-2;}
      \foreach \x in {0,1,...,\k} {
        \foreach \y in {0,1,...,#1} {
          \node[circle, fill, inner sep=1pt] at (\x, \y) {};
          \node[circle, fill, inner sep=1pt] at (\y-0.5, \x+0.5) {};
        }
      }
    
      \foreach \x in {0,1,...,\m} {
        \draw[thick] (\x, #1) -- (\x+1, #1);
        \draw[thick] (#1-0.5, \x+0.5) -- (#1-0.5, \x+1.5);
        \draw[thick] (-0.5, \x+0.5) -- (-0.5, \x+1.5);
        \foreach \y in {0,1,...,\k} {
          \draw[thick] (\x, \y) -- (\x+1, \y);
          \draw[thick] (\y-0.5, \x+0.5) -- (\y-0.5, \x+1.5);
          \draw[thick] (\x, \y) -- (\x-0.5, \y+0.5);
          \draw[thick] (\x, \y) -- (\x+0.5, \y+0.5);
          \draw[thick] (\x+0.5, \y+0.5) -- (\x, \y+1);
          \draw[thick] (\x+0.5, \y+0.5) -- (\x+1, \y+1);
        }
      }
      \foreach \y in {0,1,...,\k} {
          \draw[thick] (#1-1, \y) -- (#1-1.5, \y+0.5);
          \draw[thick] (#1-1, \y) -- (#1-0.5, \y+0.5);
          \draw[thick] (-0.5, \y+0.5) -- (0, \y+1);
          \draw[thick] (#1-0.5, \y+0.5) -- (#1-1, \y+1);
          \draw[thick] (-0.5, \y+0.5) -- (0, \y+1);
          \draw[thick] (\y-0.5, #1-0.5) -- (\y, #1);
        }
        \foreach \x in {0,1,...,#1} {
            \draw[line width=0.5mm, red] (-0.5+0.5*\x, 0.5*#1-0.5*\x) -- (0.5*#1-0.5+0.5*\x, #1-0.5*\x);
            \draw[line width=0.5mm, red] (-0.5+0.5*\x, 0.5*#1+0.5*\x) -- (0.5*#1-0.5+0.5*\x, 0.5*\x);
        }
    \end{tikzpicture}
}

\newcommand{\ie}{\textit{i}.\textit{e}. }
\newcommand{\eg}{\textit{e}.\textit{g}.\ }

\DeclareMathOperator*{\argmax}{argmax}

\usepackage[xcolor]{mdframed}
\definecolor{lightblue}{rgb}{0.96, 0.99, 0.99}
\definecolor{darkblue}{rgb}{0.7,0.81,0.87}
\definecolor{brun}{rgb}{0.87,0.72,0.53}
\definecolor{chamois}{rgb}{1.0,0.90,0.70}
\definecolor{darkpurple}{rgb}{0.81,0.7,0.87}
\definecolor{lightpurple}{rgb}{0.99, 0.96, 0.99}
\definecolor{lightbrown}{rgb}{0.99, 0.99, 0.96}
\definecolor{darkbrown}{rgb}{0.87,0.81,0.7}
\definecolor{dark}{rgb}{0.81,0.87,0.7}

\begin{document}

\title[A Complexity Map of Probabilistic Reasoning for Neurosymbolic Classification Techniques]{A Complexity Map of Probabilistic Reasoning for Neurosymbolic Classification Techniques}


\author*[1,2]{\fnm{Arthur} \sur{Ledaguenel}}\email{arthur.ledaguenel@irt-systemx.fr}

\author[2]{\fnm{Céline} \sur{Hudelot}}

\author[1]{\fnm{Mostepha} \sur{Khouadjia}}

\affil[1]{\orgname{IRT SystemX}, \orgaddress{\city{Palaiseau}, \country{France}}}

\affil[2]{\orgdiv{MICS}, \orgname{CentraleSupélec}, \orgaddress{\state{Saclay}, \country{France}}}




\abstract{Neurosymbolic artificial intelligence is a growing field of research aiming to combine neural network learning capabilities with the reasoning abilities of symbolic systems. Informed multi-label classification is a sub-field of neurosymbolic AI which studies how to leverage prior knowledge to improve neural classification systems. Recently, a family of neurosymbolic techniques for informed classification based on probabilistic reasoning has gained significant traction. Unfortunately, depending on the language used to represent prior knowledge, solving certain probabilistic reasoning problems can become prohibitively hard when the number of classes increases. Therefore, the asymptotic complexity of probabilistic reasoning is of cardinal importance to assess the scalability of such techniques. In this paper, we develop a unified formalism for four probabilistic reasoning problems. Then, we compile several known and new tractability results into a single complexity map of probabilistic reasoning. We build on top of this complexity map to characterize the domains of scalability of several techniques. We hope this work will help neurosymbolic AI practitioners navigate the scalability landscape of probabilistic neurosymbolic techniques.}

\keywords{Neurosymbolic, Probabilistic Reasoning, Computational complexity}



\maketitle

\section{Introduction} \label{sec:intro}
Neurosymbolic Artificial Intelligence (NeSy AI) is a growing field of research aiming to combine neural network learning capabilities with the reasoning abilities of symbolic systems. This hybridization can take many shapes depending on how the neural and symbolic components interact \cite{Kautz2022T,wang2023dataand}. An important sub-field of neurosymbolic AI is Informed Machine Learning \cite{VonRueden2023}, which studies how to leverage background knowledge to improve neural systems. In particular, \textbf{informed classification} studies multi-label classification tasks where prior knowledge specifies which combinations of labels are semantically valid.


A specific family of neurosymbolic techniques for informed classification, that has gained significant traction in the recent literature, interprets the output of a neural network as independent probabilities on output variables and leverages \textbf{probabilistic reasoning} to integrate prior knowledge \cite{Deng2014,Xu2018,Yang2020,Manhaeve2021,Ahmed2022spl,Ahmed2022nesyer}. Doing so, they rely on solving several probabilistic reasoning problems such as computing the probability of a propositional theory to be satisfied, a counting problem known as Probabilistic Query Evaluation (\texttt{PQE}) \cite{Suciu2020}, or finding the most probable assignment accepted by a propositional theory, an optimization problem known as Most Probable Explanation (\texttt{MPE}) \cite{Kwisthout2011}.

In this context, the \textbf{asymptotic complexity} of probabilistic reasoning is of cardinal importance to assess the \textbf{scalability} of these techniques on large classification tasks. In particular, we are interested in knowing how the complexity of a particular technique will evolve when the number of classes increases, as it is not uncommon for multi-label classification tasks to include thousands of classes (\eg ImageNet dataset \cite{Russakovsky2015} contains 1,000 classes, and up to 1,860 when adding parent classes in the WordNet hierarchy \cite{Miller1995}, Census Cars dataset \cite{Gebru2017} contains 2,675 classes of cars and iNaturalist dataset \cite{iNaturalist} contains 5,089 classes of natural species). However, most papers in the field focus on performance metrics and use toy datasets where complexity issues are not yet relevant (\eg \cite{Xu2018} tackles with simple paths in $4 \times 4$ grids, preference rankings over $4$ classes or classification with $10$ mutually exclusive classes like MNIST \cite{LeCun1998}, Fashion-MNIST \cite{Xiao2017} or Cifar-10 \cite{Krizhevsky2009}). Therefore, issues regarding computational complexity are rarely tackled in the neurosymbolic literature. This can lead to misconceptions regarding the computational limits of a given technique. For instance, \cite{Krieken2022} highlights scalability issues of existing implementations of neurosymbolic techniques on the multi-digit MNIST-addition task \cite{Manhaeve2021} and introduces an approximate method to overcome these issues. Then, \cite{Maene2023} later shows that a different encoding of the task provides a linear time (in the number of digits) computational scheme. We believe that the neurosymbolic community would benefit from a systematic study of the computational complexity of probabilistic reasoning. We hope this work will help to fill this gap.

The asymptotic complexity of probabilistic reasoning depends on the type of problems being tackled and the \textbf{representation language} used for the prior knowledge. In this paper, we draft a \textbf{complexity map} of probabilistic reasoning: we give tractability and intractability results of several probabilistic reasoning problems for different representation languages. In particular, we focus on succinct languages to guarantee the scalability of neurosymbolic techniques on specific classes of prior knowledge: hierarchical, cardinal, simple paths and matching constraints.

Finally, most probabilistic neurosymbolic techniques found in the literature rely on knowledge compilation to fragments of boolean circuits \cite{Darwiche2002} (specifically in Decomposable Negation Normal Form \texttt{DNNF} or in deterministic Decomposable Negation Normal Form \texttt{d-DNNF}) to perform probabilistic reasoning. We discuss the benefits and limits of this approach: in particular, we show that \texttt{DNNF} and its fragments do not cover the full range of scalability for probabilistic neurosymbolic techniques.


We start with preliminary definitions on graphs, knowledge representation languages, probabilistic reasoning and neurosymbolic techniques in Section \ref{sec:prem}. We characterize the conditions for scalability of probabilistic neurosymbolic techniques in Section \ref{sec:scalability}. Then, we examine the benefits and limits of knowledge compilation in Section \ref{sec:kc}. Finally, we analyze in Section \ref{sec:cmap} the asymptotical complexity of probabilistic reasoning for several succinct languages, which represent types of prior knowledge commonly used in the neurosymbolic literature. Full proofs of our results can be found in Section \ref{sec:proofs}. We mention related work in Section \ref{sec:related} and conclude with possible future research questions in Section \ref{sec:conclusion}.

Our contributions are the following. First, to the best of our knowledge, we draft the first complexity map for probabilistic reasoning that includes counting, optimization and enumeration problems. We hope this work will help neurosymbolic AI practitioners to navigate the scalability landscape of probabilistic neurosymbolic techniques. Besides, we extend known \texttt{PQE} and \texttt{MPE} tractability results for simple path and cardinal constraints with efficient compilation algorithms to \texttt{d-DNNF}, which implies \texttt{EQE} and \texttt{ThreshEnum} tractability. Finally, we show that matching constraints are \texttt{MPE}-tractable but cannot be compiled to \texttt{DNNF}: this shows that the dominant trend of compiling input theories to \texttt{DNNF}, \texttt{d-DNNF}, or one of its fragments, does not cover all cases of tractability.

\section{Preliminaries} \label{sec:prem}

\subsection{Graph theory}
A \textbf{graph} $G=(V, E)$ is composed of a finite set of vertices $V$ and a set of edges $E \subset V \times V$. $G$ is \textbf{undirected} when edges are taken as sets (\ie $E \subset \{\{u, v\}| u, v \in V \}$) and is \textbf{directed} if edges are taken as ordered pairs (\ie $E \subset \{(u, v)| u, v \in V \}$, with $(u, v) \ne (v, u)$).

Let's assume a directed graph $G=(V, E)$ and a vertex $v \in V$. The set of \textbf{incoming edges} to $v$ is $E_{in}(v)=\{(u, v) \in E\}$ and the set of \textbf{outgoing edges} from $v$ is $E_{out}(v)=\{(v, u) \in E\}$. A \textbf{source} in $G$ is a vertex that has no incoming edge. A \textbf{sink} in $G$ is a vertex that has no outgoing edge.

A \textbf{closure} in $G$ is a subset of vertices $U \subset V$ with no incoming edge from remaining vertices (\ie if $(u, v) \in E$ and $v \in U$, then $u \in U$).

A \textbf{path} in $G$ is a sequence of edges $p=(e_1=(u_1, v_1), ..., e_m=(u_m, v_m)) \in E^m$ such that $\forall 1 \leq i \leq m-1, v_i = u_{i+1}$. We say that $p$ is a path from $u_1$ to $v_m$ (or $u_1-v_m$ path). A \textbf{simple} path, or sometimes a self-avoiding path, is a path $p$ such that no vertex is visited twice (\ie $u_i \ne u_j$ for $i \ne j$ and $u_1 \ne v_m$). Notice that given a set of edges in $G$, we can easily check if a simple path can be formed using these edges. Moreover, if that is the case, this simple path is unique. A \textbf{total} path is a path from a source to a sink of $G$. A \textbf{cycle} is a path $p$ that starts and end with the same vertex with no other vertex visited twice (\ie $u_i \ne u_j$ for $i \ne j$ and $u_1 = v_m$).

$G$ is \textbf{acyclic} if it has no cycle. For a directed acyclic graph $G=(V, E)$, a \textbf{topological ordering} of the vertices is a bijection $\sigma:V \mapsto \llbracket 1, |V| \rrbracket$ such that for $u, v \in V$, if there is a path from $u$ to $v$ then $\sigma(u) < \sigma(v)$. A \textbf{topological ordering} of the edges is a bijection $\sigma:E \mapsto \llbracket 1, |V| \rrbracket$ such that for $e_1=(u_1, v_1), e_2=(u_2, v_2) \in E$, if there is a path from $v_1$ to $u_2$ then $\sigma(e_1) < \sigma(e_2)$.

$G$ is a (directed) \textbf{tree} if it has a single source $r$, called a \textbf{root}, and for any other vertex $v$ there is a unique path from $r$ to $v$. A sink in a tree is called a \textbf{leaf}.

A \textbf{matching} in an undirected graph $G=(V, E)$ is a set of edges $M \subset E$ such that no vertex is covered twice (\ie $\forall e_1, e_2 \in M, e_1 \neq e_2 \iff e_1 \cap e_2 = \emptyset$).


\subsection{Knowledge representation} \label{sec:kr}
In its more abstract form, \textbf{knowledge} about a \textbf{world} tells us in what \textbf{states} this world can be observed. In this paper, we only consider propositional knowledge, where the states correspond to subsets of a discrete set of variables $\mathbf{Y}$ and knowledge tells us what combinations of variables can be observed in the world. The set of \textit{possible} states is $\mathbb{B}^{\mathbf{Y}}$, where $\mathbb{B}:= \{0, 1\}$ is the set of boolean values. A state $\mathbf{y} \in \mathbb{B}^{\mathbf{Y}}$ can be seen as a subset of $\mathbf{Y}$ as well as an application that maps each variable to $\mathbb{B}$ (\ie for a variable $Y_i \in \mathbf{Y}$, $\mathbf{y}_i=1$ is equivalent to $Y_i \in \mathbf{y}$). Knowledge defines a set of states that are considered \textit{valid}. An \textit{abstract} representation of this knowledge is a \textbf{boolean function} $f \in \mathbb{B}^{\mathbb{B}^{\mathbf{Y}}}$, which can be viewed either as a function that maps states in $\mathbb{B}^{\mathbf{Y}}$ to $\mathbb{B}$ or as a subset of $\mathbb{B}^{\mathbf{Y}}$. However, in order to exploit this knowledge (\eg reason, query, communicate, etc.), we need a \textit{concrete} language to represent it.

A language for representing knowledge has two sides. The \textbf{syntax} defines admissible \textbf{statements} that can be made about the world. The \textbf{semantic} determines the relation between statements and states: it specifies in which states a statement can be considered \textit{true} or \textit{false}, or conversely for which statements a state is considered \textbf{valid}.

Knowledge representation languages for boolean functions come from diverse fields and are often designed to meet specific needs. For instance, propositional logic comes from the field of logic and the \texttt{SAT} community, boolean circuits and decision diagrams \cite{Darwiche2002} come from the fields of knowledge representation and automated reasoning, while random forests \cite{breiman_random_2001}, boosted trees \cite{Freund1995} or binarized neural networks \cite{hubara_binarized_2016} come from the field of machine learning.

To present a unified view, we propose the following definition of a \textbf{propositional language}, inspired by the work in \cite{Brewka2007,Lierler2024}.

\begin{definition}[Propositional language] \label{def:kr}
    A \textbf{propositional language} is a couple $\mathtt{L}:=(\mathtt{T}, \mathscr{s})$ such that for any discrete set of variables $\mathbf{Y}$:
        \begin{itemize}
            \item the \textbf{syntax} $\mathtt{T}$ defines a set of sentences $\mathtt{T}(\mathbf{Y})$ called \textbf{theories} on $\mathbf{Y}$ and written from $\mathbf{Y}$ and a finite set of symbols specific to $\mathtt{L}$.
            \item the \textbf{semantics} $\mathscr{s}$ maps a theory on $\mathbf{Y}$ to a boolean function in $\mathbb{B}^{\mathbb{B}^{\mathbf{Y}}}$:
                \begin{equation*}
                    \mathscr{s}(\mathbf{Y}) : \mathtt{T}(\mathbf{Y}) \rightarrow \mathbb{B}^{\mathbb{B}^{\mathbf{Y}}}
                \end{equation*}
        \end{itemize}
\end{definition}

When the set of variables is clear from context, we simply note $\kappa \in \mathtt{T}$ and $\mathscr{s}(\kappa)$ in place of $\kappa \in \mathtt{T}(\mathbf{Y})$ and $\mathscr{s}(\mathbf{Y})(\kappa)$. 

\paragraph{Satisfiability}

A state $\mathbf{y} \in \mathbb{B}^{\mathbf{Y}}$ \textbf{satisfies} a theory $\kappa\in \mathtt{T}$ iff it belongs to the boolean formula represented by $\kappa$ (\eg $\mathbf{y} \in \mathscr{s}(\mathbf{Y})(\kappa)$). We also say that $\kappa$ accepts $\mathbf{y}$. A theory $\kappa$ is \textbf{satisfiable} if it is satisfied by a state, \ie if $\mathscr{s}(\mathbf{Y})(\kappa) \ne \emptyset$.

\paragraph{Equivalence}

Two theories $\kappa_1$ and $\kappa_2$ are \textbf{equivalent} iff they represent the same boolean function (\eg $\mathscr{s}(\kappa_1) = \mathscr{s}(\kappa_2)$). This notion can be extended to two theories belonging to distinct propositional languages, which is key for knowledge compilation (see Section \ref{sec:kc}).

\paragraph{Fragments}
We say that a propositional language $\mathtt{L}_2:=(\mathtt{T}_2, \mathscr{s}_2)$ is a \textbf{fragment} of a propositional language $\mathtt{L}_1:=(\mathtt{T}_1, \mathscr{s}_1)$, noted $\mathtt{L}_2 \subset \mathtt{L}_1$, iff for any discrete set of variables $\mathbf{Y}$: $\mathtt{T}_2(\mathbf{Y}) \subset \mathtt{T}_1(\mathbf{Y})$ and $\mathscr{s}_2(\mathbf{Y})$ is the restriction of $\mathscr{s}_1(\mathbf{Y})$ to $\mathtt{T}_2(\mathbf{Y})$.

\paragraph{Completeness}

A propositional language is \textbf{complete} iff, for any discrete set of variables $\mathbf{Y}$, any boolean function $\mathscr{f} \in \mathbb{B}^{\mathbb{B}^{\mathbf{Y}}}$ can be represented by a theory, \ie $\forall \mathscr{f} \in \mathbb{B}^{\mathbb{B}^{\mathbf{Y}}}, \exists \kappa \in \mathtt{T}(\mathbf{Y}), \mathscr{s}(\kappa) = \mathscr{f}$.

\paragraph{Size}

The \textbf{size} of a theory $\kappa$, noted $|\kappa|$, is its length as a sentence. However, the syntax is sometimes easier to represent using set of sentences or graphs, in which case we detail for each language the appropriate size measure.

\vspace{5mm}
As we illustrate below, Definition \ref{def:kr} covers many knowledge representation languages commonly found in the neurosymbolic literature. We give on Table \ref{tab:lang_acro} a list of the main acronyms we will use to designate propositional languages throughout the paper.

\begin{table}[h]
    \centering
    \begin{tabular}{r l}
        \hline
        \textbf{Acronyms} & \textbf{Languages} \\
        \hline
        \multicolumn{2}{l}{\textbf{Complete}} \\
        \texttt{PL} & Propositional Logic \\
        \texttt{CNF} & Conjunctive Normal Form \\
        \texttt{BC} & Boolean Circuits \\
        \texttt{DNNF} & Decomposable Negation Normal Form \\
        \texttt{d-DNNF} & deterministic Decomposable Negation Normal Form \\
        \texttt{OBDD} & Ordered Binary Decision Diagram \\
        \texttt{BLP} & Binary Linear Programming \\
        \\
        \multicolumn{2}{l}{\textbf{Succinct}} \\
        \texttt{Hex} & Hierarchical and exclusion constraints \\
        \texttt{H} & Hierarchical constraints \\
        \texttt{T-H} & Hierarchical constraints on a tree \\
        \texttt{E-H} & Hierarchical constraints with assumed exclusions \\
        \texttt{TE-H} & Hierarchical constraints on a tree with assumed exclusions \\
        \texttt{Card} & Cardinal constraints \\
        \texttt{SPath} & Simple path constraints \\
        \texttt{ASPath} & Simple path constraints on acyclic graphs \\
        \texttt{Match} & Matching constraints \\
    \hline
    \end{tabular}
    \caption{Acronyms used for propositional languages.}
    \label{tab:lang_acro}
\end{table}

\paragraph{Propositional logic} Propositional logic $\mathtt{PL}:=(\mathtt{T}_{PL}, \mathscr{s}_{PL})$ is the most common propositional language, typically used as an introduction to logic and knowledge representation in many textbooks \cite{Russell2021}.

A theory $\kappa \in \mathtt{T}_{PL}(\mathbf{Y})$ is called a \textbf{propositional formula} and is formed inductively from variables and other formulas by using unary ($\neg$, which expresses negation) or binary ($\lor, \land$, which express disjunction and conjunction respectively) connectives:
\begin{gather*}
        \phi := \quad v \quad | \quad  \neg \phi \quad | \quad \phi \land \varphi \quad | \quad \phi \lor \varphi, \\
        v \in \mathbf{Y}, \phi, \varphi \in \mathtt{T}_{PL}(\mathbf{Y})
\end{gather*}

A variable or its negation is called a \textbf{literal} (respectively a \textbf{positive} or a \textbf{negative} literal), a disjunction of literals $\bigvee_i l_i$ is called a \textbf{clause}, a conjunction of literals $\bigwedge_i l_i$ is called a \textbf{term}. A formula is in Conjunctive Normal Form (\texttt{CNF}) if it is a conjunction of clauses $\bigwedge_i \bigvee_j l_{i, j}$. A formula is in Disjunctive Normal Form (\texttt{DNF}) if it is a disjunction of terms $\bigvee_i \bigwedge_j l_{i, j}$. A \texttt{CNF} is in \texttt{Horn} if each of its clauses contains at most one positive literal. A \texttt{CNF} is in 2-\texttt{CNF} if each of its clauses contains at most two literals. A formula is in 2-\texttt{Horn} if it is both in \texttt{Horn} and 2-\texttt{CNF}. A \texttt{CNF} is \textbf{monotone} (resp. \textbf{negative}) if each of its clauses contains only positive (resp. negative) literals. The set of monotone (resp. negative) 2-\texttt{CNF} is noted \texttt{monotone 2-CNF} (resp. \texttt{neg 2-CNF}). Finally, a \texttt{DNF} is in \texttt{MODS} if each of its terms contains every variable either as a positive or a negative literal (\ie each term accepts a single state).

\begin{remark}
    The \texttt{MODS} fragment of propositional logic is the closest to the boolean function representation since all accepted states are extensively represented.
\end{remark}

The semantics of propositional logic can be inductively derived from the formula following the standard semantics of negation, conjunction and disjunction, \ie a state $\mathbf{y}$ satisfies:
\begin{itemize}
    \item a variable $Y_i \in \mathbf{Y}$ if $y_i=1$
    \item a formula $\neg \phi$ if $\mathbf{y}$ does not satisfy $\phi$
    \item a formula $\phi \lor \psi$ if $\mathbf{y}$ satisfies $\phi$ or $\psi$
    \item a formula $\phi \land \psi$ if $\mathbf{y}$ satisfies $\phi$ and $\psi$
\end{itemize}

\begin{example}
    $Y_1 \land Y_2$ is satisfied by $\mathbf{y}$ iff $Y_1$ and $Y_2$ are satisfied by $\mathbf{y}$, which means iff $y_1=y_2=1$.
\end{example}

\paragraph{Boolean circuits} Boolean circuits \cite{Darwiche2002} $\mathtt{BC}:=(\mathtt{T_{BC}}, \mathscr{s}_{C})$ is a representation language that has gained a lot of traction in recent years because some of its fragments provide tractable algorithms of many reasoning tasks.

A \textbf{boolean circuit} $C \in \mathtt{T_{BC}}(\mathbf{Y})$ is a couple $C:=(G, \varsigma)$ where:
    \begin{itemize}
        \item $G=(N, W)$ is a directed acyclic graph
        \item vertices in $N$ are called \textbf{nodes} (or sometimes \textbf{gates}) and edges in $W$ are called \textbf{wires}
        \item $G$ has a single root $r$ (\ie a node without incoming wire)
        \item $\varsigma: N \rightarrow \mathbf{Y} \cup \{\top, \bot, \neg, \land, \lor\}$ such that:
        \begin{itemize}
            \item $\varsigma(n) \in \mathbf{Y} \cup \{\top, \bot\}$ iff $n$ is a leaf
            \item $\varsigma(n)=\neg$ iff $n$ has exactly one child
            \item $\varsigma(n) \in \{\land, \lor\}$ iff $n$ has at least two children
        \end{itemize}
    \end{itemize}

\begin{figure}[t]
    \centering
    \includegraphics[width=0.6\linewidth]{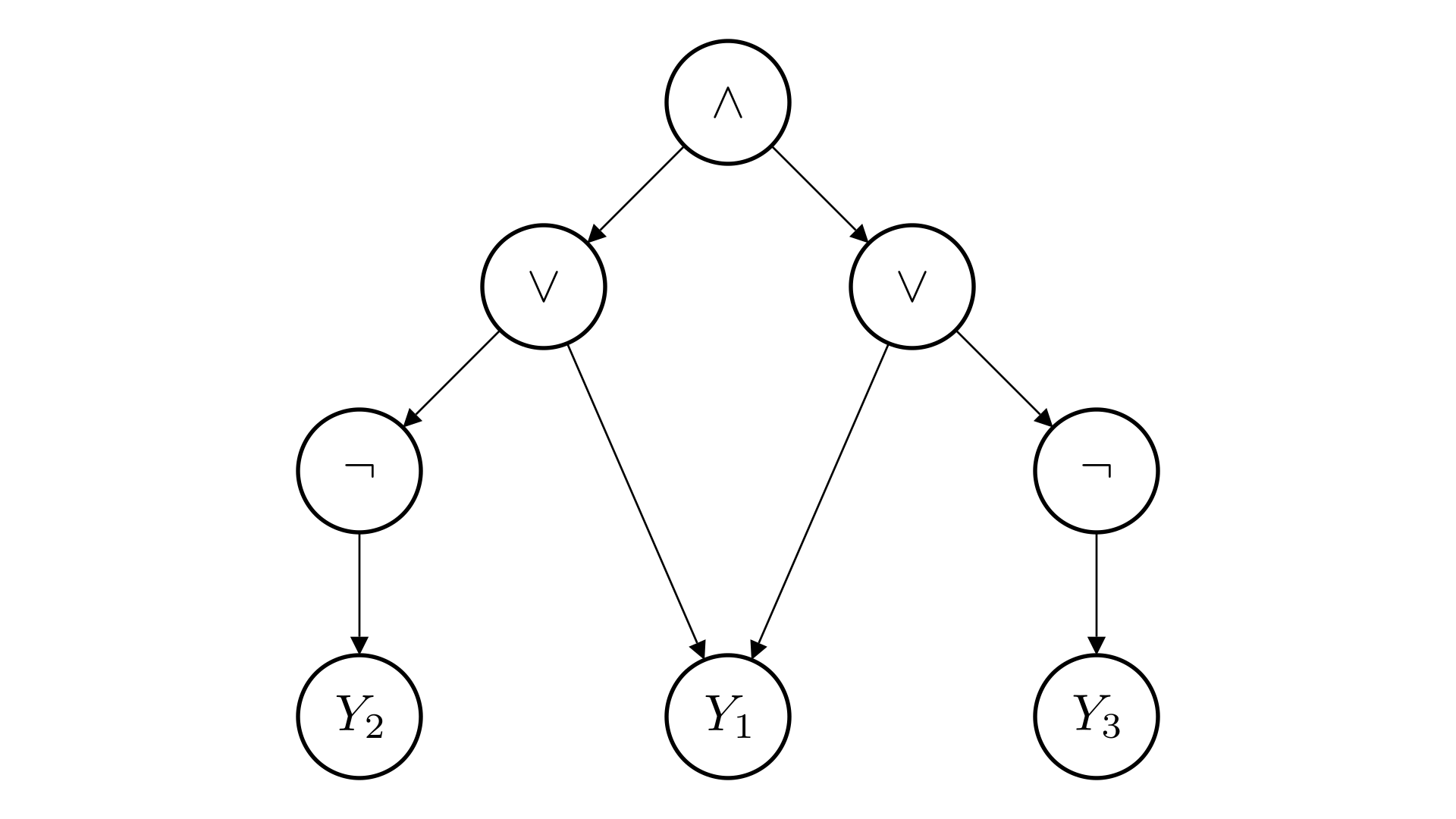}
    \caption{A boolean circuit.}
    \label{fig:circuit}
\end{figure}
    
The size of a circuit is its number of wires, \ie $|C|=|W|$. The set of children of a node $n \in N$ is noted $ch(n)$. The set of variables of a circuit is noted $var(C)$. Given a node $n \in N$, we note $C^n$ the circuit obtained by keeping all nodes that are descendants of $n$ in $G$. We sometimes note $var(n)$ for $var(C^n)$.

\vspace{5mm}
Let's assume a state $\mathbf{y} \in \mathbb{B}^{\mathbf{Y}}$ and a circuit $C:=(G, \varsigma) \in \mathtt{T_{BC}}(\mathbf{Y})$ of root $r$. To know if $\mathbf{y}$ satisfies $C$, we evaluate the circuit bottom up, mapping each node $n$ to $0$ or $1$. First, leaf nodes $n$ are mapped to $1$ if $\varsigma(n)=\top$, to $0$ if $\varsigma(n)=\bot$ and to $\mathbf{y}(\varsigma(n))$ if $\varsigma(n) \in \mathbf{Y}$. Then, for any internal node $n$, it is valued $1$ if $\varsigma(n)=\neg$ and its child is valued $1$, or if $\varsigma(n)=\lor$ and one of its children is valued $1$ or if $\varsigma(n)=\land$ and all its children are valued $1$. Otherwise it is valued $0$. The state $\mathbf{y}$ satisfies the circuit if the root is valued at $1$.

A circuit is in Negation Normal Form (\texttt{NNF}) if all negation nodes have a variable node as a child. A $\land$-node $u$ is \textbf{decomposable} if the sub-circuits rooted in each of its children do not share variables. A circuit is in Decomposable Negation Normal Form (\texttt{DNNF}) if it is \texttt{NNF} and all of its $\land$-nodes are decomposable. A $\lor$-node $u$ is \textbf{deterministic} if the sub-circuits rooted in each of its children do not share satisfying states. A circuit is in deterministic Decomposable Negation Normal Form (\texttt{d-DNNF}) if it is \texttt{DNNF} and all of its $\lor$-nodes are deterministic.

Any propositional formula can be translated in linear time into an equivalent boolean circuit by reading the formula in the standard priority order. Therefore, usual fragments of propositional logic (\eg \texttt{CNF}, \texttt{DNF}, etc.) translate into fragments of boolean circuits. Likewise, decision diagrams like Ordered Binary Decision Diagrams (\texttt{OBDD}) \cite{Bryant1986} also correspond to fragments of boolean circuits \cite{Amarilli24}.

\vspace{5mm} 
\begin{example}
    The circuit represented on Figure \ref{fig:circuit} is equivalent to the \texttt{CNF} formula $(Y_1 \lor \neg Y_2) \land (Y_1 \lor \neg Y_3)$. It is in \texttt{NNF} but is neither decomposable (the two children of the root $\land$-node share variable $Y_1$) nor deterministic (both $\lor$-nodes are not deterministic).
\end{example}

\vspace{5mm} 
In this paper, we adopt a different graphical representation for our circuits, which was designed for \texttt{SDD} in \cite{Darwiche2011} and is well suited to the kind of circuits we need to represent. The circuit represented on Figure \ref{fig:node} is equivalent to $(\alpha_1 \land \beta_1) \lor (\alpha_2 \land \beta_2)$ where $\alpha_1, \beta_1, \alpha_2$ are terminal nodes (\ie literals, $\top$ or $\bot$) and $\beta_2$ is a sub-circuit.

\begin{figure}[t]
    \centering
    \includegraphics[width=0.3\linewidth]{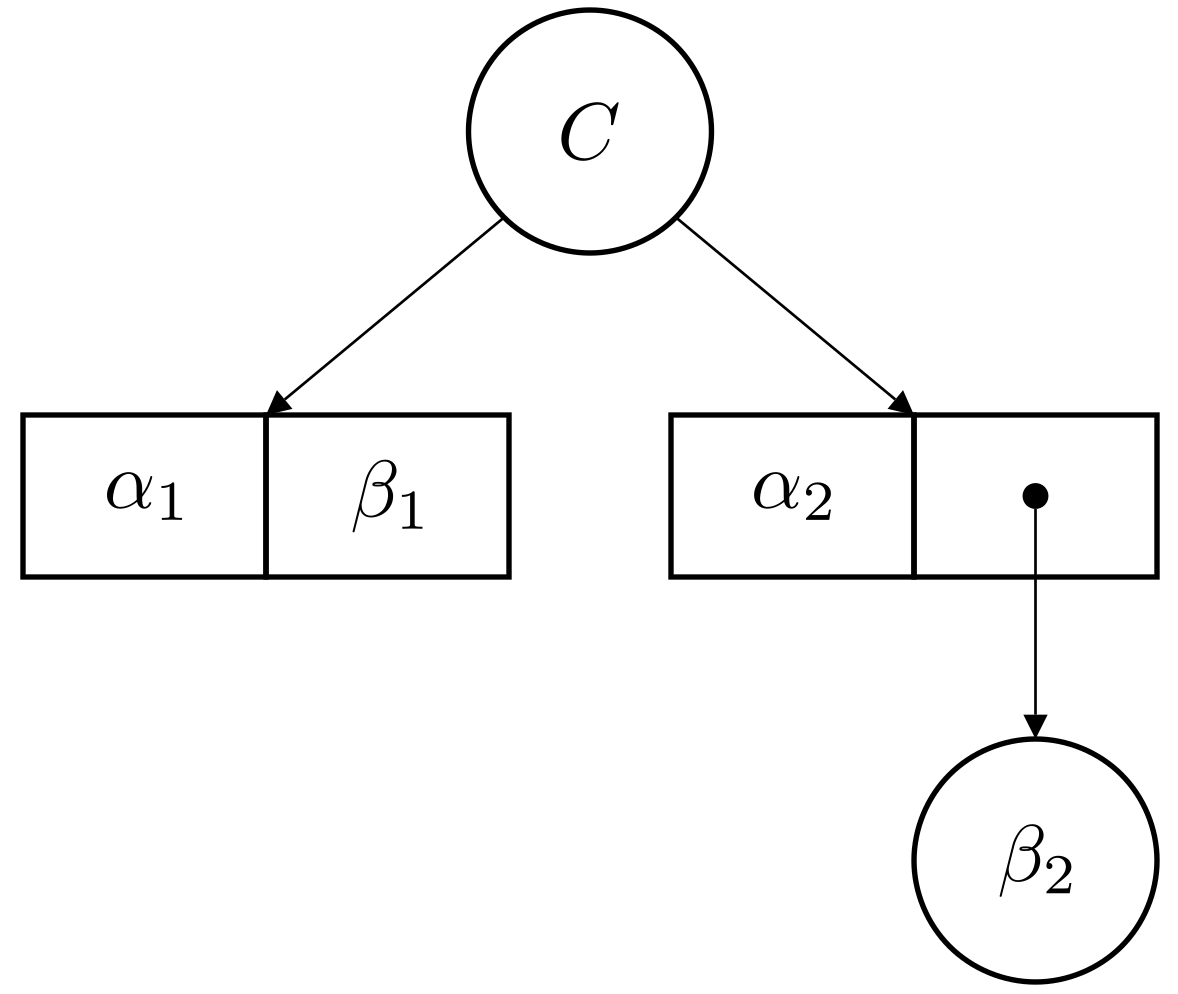}
    \caption{Illustration of a boolean circuit $C$ following the graphical representation drawn from \cite{Darwiche2011}.}
    \label{fig:node}
\end{figure}

\paragraph{Binary Linear Programming} Binary Linear Programming is traditionally associated to constrained optimization problems, but it can be used to define a propositional language $\mathtt{BLP}:=(\mathtt{T_{BLP}}, \mathscr{s}_{\mathtt{BLP}})$ naturally suited to express many real-world problems. This language is also known as conjunction of pseudo-boolean constraints (\texttt{PBC}) \cite{LeBerre2018}.

A theory $\Pi \in \mathtt{T_{BLP}}(\mathbf{Y})$, called a \textbf{linear program}, is a set of formulas called \textbf{linear constraints}. A linear constraint $r$ on variables $\mathbf{Y}$ is of the shape:
    \begin{equation*}
        r: b_1.Y_{i_1} + ... + b_m.Y_{i_m} \leq c
    \end{equation*}
with $Y_{i_1},..., Y_{i_m} \in \mathbf{Y}$ and $b_1,..., b_m, c \in \mathbb{Z}$.

We use relative weights on variables instead of positive weights on literals as in \cite{LeBerre2018}. The two notations are strictly equivalent but we prefer relative weights on variables because it allows to define our \texttt{Card} fragment more easily (see Section \ref{sec:card}).

For a more rigorous definition, relative integers in $\mathbb{Z}$ are represented in binary (\ie linear constraints are written with symbols in $\mathbf{Y} \cup \{0,1,.,+,-,\leq\}$), which means that the size of a linear constraint takes into account the size of the binary representation of its coefficients. Therefore, the size of a linear constraint $r:b_1.Y_{i_1} + ... + b_m.Y_{i_m} \leq c$ is $\ulcorner \log(c) \urcorner + \sum_i \ulcorner log(b_i) \urcorner$ and the size of a linear program is the sum of sizes of its linear constraints. Besides, the space of coefficients $\mathbb{Z}$ can be extended to $\mathbb{Q}$ without affecting either the succinctness of the language nor its expressivity. However, it cannot be extended to $\mathbb{R}$ because arbitrary irrational coefficients would require an infinite size to represent.

To lighten notations, a linear constraint is sometimes noted $r:\langle \mathbf{Z}, \mathbf{b} \rangle \leq c$ where $\mathbf{Z}:= (Y_{i_1},..., Y_{i_m})$ and $\mathbf{b}:= (b_1,..., b_m)$. We can also use the symbol $\geq$ to note $\langle \mathbf{Z}, \mathbf{b} \rangle \geq c$ instead of $\langle \mathbf{Z}, -\mathbf{b} \rangle \leq -c$ and the symbol $=$ to note $\langle \mathbf{Z}, \mathbf{b} \rangle = c$ in place of the two linear constraints $\langle \mathbf{Z}, \mathbf{b} \rangle \leq c$ and $\langle \mathbf{Z}, \mathbf{b} \rangle \geq c$.

\vspace{5mm}            
A state $\mathbf{y} \in \mathbb{B}^{\mathbf{Y}}$ \textbf{satisfies} a linear constraint $r: b_1.Y_{i_1} + ... + b_m.Y_{i_m} \leq c$ iff $b_1.y_{i_1} + ... + b_m.y_{i_m} \leq c$ in the usual arithmetical sense. A state $\mathbf{y} \in \mathbb{B}^{\mathbf{Y}}$ satisfies a linear program $\Pi \in \mathtt{T_{BLP}}(\mathbf{Y})$ iff it satisfies all linear constraints in $\Pi$.
\vspace{5mm}
\begin{example}
    Imagine a catalog of products $\mathbf{P}:=\{P_1, P_2, P_3, P_4\}$ with corresponding prices $\mathbf{p}:=\{10, 100, 20, 50\}$. A basket of products corresponds to a state on $\mathbf{P}$. An online website might want to suggest a basket of additional products to go with the order of a client. However, it noticed that large or expansive baskets are less likely to be picked. However, they would like to make sure that the suggested basket is not too cheap. Therefore, they defined a maximum size $3$ as well as maximum and minimum budgets $150$ and $30$ for the suggested baskets. Baskets that match those constraints are: $(P_1, P_2), (P_1, P_3), (P_1, P_4), (P_1, P_3, P_4), (P_2), (P_3, P_4)$. This set of baskets corresponds to a boolean function on $\mathbf{P}$, which can be represented by the following linear program:
        \begin{gather*}
        \Pi := \left\{
        \begin{array}{cc}
            P_1 + P_2 + P_3 + P_4 \leq 3 \\
            10 \times P_1 + 100 \times P_2 + 20 \times P_3 + 50 \times P_4 \leq 110 \\
            10 \times P_1 + 100 \times P_2 + 20 \times P_3 + 50 \times P_4 \geq 30
        \end{array}
        \right.
        \end{gather*}
\end{example}


\paragraph{Graph-based languages} Although they are not usually thought of as propositional languages, mapping variables to elements of a graph allow us to represent a boolean function as the set of substructures of the graph that verifies a specific property. For instance, we use this to introduce the \texttt{SPath} language, which maps variables to edges of a directed graph and represents the sets of edges that form a total simple path in the graph (more details in Section \ref{sec:paths}), or the \texttt{Match} language, which maps variables to edges of a undirected graph and represents the sets of edges that form a matching in the graph (more details in Section \ref{sec:match}).

A graph-based language is \textbf{edge-based} (resp. \textbf{vertex-based}) when a theory maps variables in $\mathbf{Y}$ to edges (resp. vertices) of a graph $G=(V, E)$: a theory is a couple $\kappa:=(G, \varsigma)$ where $G=(V, E)$ is a graph and $\varsigma: E \rightarrow \mathbf{Y}$ (resp. $\varsigma: V \rightarrow \mathbf{Y}$) is \textbf{bijective}. A graph-based language is \textbf{directed} (resp. \textbf{undirected}) when theories are composed of directed (resp. undirected) graphs. A directed graph-based language can also be \textbf{acyclic} if theories are composed of directed acyclic graphs. For instance, \texttt{SPath} and \texttt{Match} are two edge-based languages.

The size of a graph-based theory is taken as the size of the graph on which it is based. As opposed to most traditional propositional languages, graph-based languages are not \textbf{complete} (\ie they cannot represent any boolean function), but specialized for a specific type of knowledge. However, graphical languages are naturally \textbf{succinct} (\ie the size of a theory is polynomial in its number of variables) since the size of a graph is at most quadratic in its number of nodes or edges.


\subsection{Probabilistic reasoning} \label{sec:probs}
One challenge of neurosymbolic AI is to bridge the gap between the discrete nature of logic and the continuous nature of neural networks. Probabilistic reasoning can provide the interface between these two realms by allowing us to reason about uncertain facts. 

\paragraph{Distributions} A joint probability distribution on a set of \textbf{boolean variables} $\mathbf{Y}$ is an application $\mathcal{P}:\mathbb{B}^{\mathbf{Y}} \mapsto \mathbb{R}^+$ that maps each state $\mathbf{y}$ to a probability $\mathcal{P}(\mathbf{y})$, such that $\sum_{\mathbf{y} \in \mathbb{B}^{\mathbf{Y}}} \mathcal{P}(\mathbf{y})=1$. To define internal operations between distributions, like multiplication, we extend this definition to un-normalized distributions $\mathcal{E}:\mathbb{B}^{\mathbf{Y}} \mapsto \mathbb{R}^+$. The \textbf{partition function} $\mathsf{Z}:\mathcal{E} \mapsto \sum_{\mathbf{y} \in \mathbb{B}^{\mathbf{Y}}} \mathcal{E}(\mathbf{y})$ maps each distribution to its sum, and we note $\overline{\mathcal{E}} := \frac{\mathcal{E}}{\mathsf{Z}(\mathcal{E})}$ the normalized distribution (when $\mathsf{Z}(\mathcal{E}) \ne 0$). The \textbf{entropy function} $\mathsf{H}:\mathcal{P} \mapsto \sum_{\mathbf{y} \in \mathbb{B}^{\mathbf{Y}}} - \mathcal{P}(\mathbf{y}) \cdot \log(\mathcal{P}(\mathbf{y}))$ maps each probability distribution to its entropy. The \textbf{mode} of a distribution $\mathcal{E}$ is its most probable state, ie $\underset{\mathbf{y} \in \mathbb{B}^{\mathbf{Y}}}{\argmax}\mathcal{E}(\mathbf{y})$.

A standard hypothesis for joint distributions of boolean variables is \textbf{independence}. Assume a vector of probabilities $\mathbf{p}=(p_i)_{1 \leq i \leq k} \in [0, 1]^k$, one for each variable, the joint distribution of independent Bernoulli variables $\mathcal{B}(p_i)_{1 \leq i \leq k}$ is the distribution $\mathcal{P}(\cdot | \mathbf{p})$ such that:
    \begin{equation}
        \mathcal{P}(\cdot | \mathbf{p}):  \mathbf{y} \mapsto \prod_{\substack{1 \leq i\leq k \\ y_i=1}} p_i \times \prod_{\substack{1 \leq i\leq k \\ y_i=0}} (1-p_i)
    \end{equation}

    

\paragraph{Reasoning} Typically, when belief about random variables is expressed through a probability distribution and new information is collected in the form of evidence (\ie a partial assignment of the variables), we are interested in two things: computing the probability of such evidence and updating our beliefs using Bayes' rules by conditioning the distribution on the evidence. Probabilistic reasoning allows us to perform the same operations with logical knowledge in place of evidence.

Let's assume a probability distribution $\mathcal{P}$ on variables $\mathbf{Y}:= \{Y_j\}_{1\leq j \leq k}$ and a \textbf{satisfiable} theory $\kappa$ from a propositional language $\mathtt{L}:=(\mathtt{T}, \mathscr{s})$. Notice that the boolean function $\mathscr{s}(\kappa)$ represented by $\kappa$ is an un-normalized distribution on $\mathbf{Y}$.
\vspace{5mm}
\begin{definition}
    The \textbf{probability} of $\kappa$ under $\mathcal{P}$ is:
    \begin{equation}
        \mathcal{P}(\kappa) := \mathsf{Z}(\mathcal{P} \cdot \mathscr{s}(\kappa)) = \sum_{\mathbf{y} \in \mathbb{B}^{\mathbf{Y}}} \mathcal{P}(\mathbf{y}) \cdot \mathscr{s}(\kappa)(\mathbf{y})
    \end{equation}

    The distribution $\mathcal{P}$ \textbf{conditioned on} $\kappa$, noted $\mathcal{P}(\cdot | \kappa)$, is:
    \begin{equation}
        \mathcal{P}(\cdot | \kappa):= \overline{\mathcal{P} \cdot \mathscr{s}(\kappa)}
    \end{equation}
\end{definition}


For the remainder of the paper, we assume non-trivial (\ie $\forall i, p_i \notin \{0, 1\}$) and rational probabilities (\ie $\forall i, p_i \notin \{0, 1\}$ and $p_i \in \mathbb{Q}$). This implies that $\mathcal{P}(\cdot | \mathbf{p})$ is strictly positive (\ie takes a strictly positive value for each state). Since $\kappa$ is satisfiable, we can properly define:
    \begin{equation*}
        \mathcal{P}(\kappa | \mathbf{p}):=\mathsf{Z}(\mathcal{P}(\cdot | \mathbf{p}) \cdot \mathscr{s}(\kappa)) \qquad \qquad \mathcal{P}(\cdot | \mathbf{p}, \kappa):=\frac{\mathcal{P}(\cdot | \mathbf{p}) \cdot \mathscr{s}(\kappa)}{\mathcal{P}(\kappa| \mathbf{p})}
    \end{equation*}

Performing those operations or computing other quantities relative to conditioned distributions constitute probabilistic reasoning problems. We describe below several probabilistic reasoning problems, divided into three types: counting, optimization and enumeration problems. A list of acronyms used throughout the paper for reasoning problems is given on Table \ref{tab:pr_acro}. Solving these problems is at the core of many neurosymbolic techniques, as shown in Section \ref{sec:techniques}.

\begin{table}[h]
    \centering
    \begin{tabular}{r l}
        \hline
        \textbf{Acronyms} & \textbf{Problems} \\
        \hline
        \multicolumn{1}{l}{\textbf{Classification}} &  \\
        \texttt{MC} & Model Counting \\
        \texttt{PQE} & Probabilistic Query Evaluation \\
        \texttt{EQE} & Entropy Query Evaluation \\
        \\
        \multicolumn{1}{l}{\textbf{Optimization}} &  \\
        \texttt{MPE} & Most Probable Explanation \\
        \texttt{top-k} & Top-k States \\
        \\
        \multicolumn{1}{l}{\textbf{Enumeration}} &  \\
        \texttt{ThreshEnum} & Threshold Enumeration \\
        \hline
    \end{tabular}
    \caption{Acronyms used for reasoning problems.}
    \label{tab:pr_acro}
\end{table}

\paragraph{Counting} Computing $\mathcal{P}(\kappa | \mathbf{p})$ is a \textbf{counting} problem called \textbf{Probabilistic Query Evaluation} (\texttt{PQE}) (equivalent to the Weighted Model Counting (\texttt{WMC}) problem) \cite{Suciu2020}. Another counting problem called \textbf{Entropy Query Evaluation} (\texttt{EQE}) consists in computing the entropy $\mathsf{H}(\mathcal{P}(\cdot | \mathbf{p}, \kappa))$ of the conditioned distribution. Notice that the standard reasoning problem of \textbf{Model Counting} \texttt{MC}, \ie counting the number of satisfying states for a theory $\kappa$ (sometimes noted \#\texttt{SAT} when $\kappa$ is a \texttt{CNF}), can be reduced to both \texttt{PQE} and \texttt{EQE} by using identical probabilities. Notice that computing $\mathcal{P}(\mathbf{y} | \mathbf{p}, \kappa)$ for a satisfying state $\mathbf{y} \in \mathscr{s}(\kappa)$ is equivalent to solving \texttt{PQE} because $\mathcal{P}(\mathbf{y} | \mathbf{p})$ can be computed in polynomial time and:
\begin{equation*}
    \mathcal{P}(\mathbf{y} | \mathbf{p}, \kappa) = \frac{\mathcal{P}(\mathbf{y} | \mathbf{p})}{\mathcal{P}(\kappa | \mathbf{p})}
\end{equation*}

\paragraph{Optimization} Computing the mode of $\mathcal{P}(\cdot | \mathbf{p}, \kappa)$ is an \textbf{optimization} problem called \textbf{Most Probable Explanation} (\texttt{MPE}) \cite{Kwisthout2011}. The mode of $\mathcal{P}(\cdot | \mathbf{p}, \kappa)$ is also called the maximum a posteriori (\texttt{MAP}) prediction and \texttt{MPE} is sometimes referred as \texttt{MAP} inference. This task can be extended to compute the top-k states of $\mathcal{P}(\cdot | \mathbf{p}, \kappa)$ (\texttt{top-k}). The decision problem of satisfiability \texttt{SAT} (\ie deciding if a theory $\kappa$ is satisfied by a state or not) can be reduced to \texttt{MPE}. When $\kappa$ is in \texttt{CNF}, \texttt{MPE} on $\mathcal{P}(\cdot | \mathbf{p}, \kappa)$ can be seen as a partial weighted MaxSAT \cite{Li2021} (\texttt{pw-MaxSAT}) instance: $\kappa$ constitutes the \textit{hard} part of the instance and the \textit{soft} parts is composed of weighted positives literals $(Y_i, \sigma^{-1}(p_i))$ where $\sigma^{-1}$ is the inverse of the sigmoid function $\sigma : a \mapsto \frac{e^a}{1+e^a}$. The abstract view on optimization in propositional frameworks introduced in \cite{Lierler2024} allows to extend the parallel between \texttt{MPE} and \texttt{pw-MaxSAT} to arbitrary propositional languages.

\paragraph{Enumeration} Listing the set of satisfying states in decreasing order of probability is an \textbf{enumeration} problem called \textbf{Ranked Enumeration} (\texttt{RankedEnum}). A closely related enumeration problem, which we call \textbf{Threshold Enumeration} (\texttt{ThreshEnum}), consists in listing the set of satisfying states with a probability superior to a given threshold.


\subsection{Probabilistic neurosymbolic techniques} \label{sec:techniques}

In machine learning, the objective is usually to learn a functional relationship $f:\mathcal{X} \mapsto \mathcal{Y}$ between an \textbf{input domain} $\mathcal{X}$ and an \textbf{output domain} $\mathcal{Y}$ from data samples. Multi-label classification is a type of machine learning tasks where input samples are labeled with subsets of a finite set of classes $\mathbf{Y}$. Therefore, labels can be understood as states on the set of variables $\mathbf{Y}$. In this case, the output space of the task, \ie the set of all labels, is $\mathcal{Y} = \mathbb{B}^{\mathbf{Y}}$. In \textbf{informed} classification, prior knowledge (sometimes called background knowledge) specifies which states in the output domain are semantically \textbf{valid}, \ie to which states can input samples be mapped. The set of valid states constitute a boolean function $\mathscr{f} \in \mathbb{B}^{\mathbb{B}^{\mathbf{Y}}}$ on the set of variables $\mathbf{Y}$, which can be represented as a theory $\kappa \in \mathtt{T}(\mathbf{Y})$ of a propositional language $\mathtt{L}:=(\mathtt{T}, \mathscr{s})$ such that $\mathscr{s}(\kappa) = \mathscr{f}$. For instance, hierarchical and exclusion constraints are used in \cite{Deng2014}, propositional formulas in \texttt{CNF} in \cite{Xu2018}, boolean circuits in \cite{Ahmed2022spl}, ASP programs in \cite{Yang2020} and linear programs in \cite{niepert_implicit_2021}.

Unless mentioned otherwise, we assume that a neural classification system is composed of three modules: 
\begin{itemize}
    \item a parametric and differentiable \textbf{model}, which takes in an input sample $x \in \mathcal{X}$ and produces a vector of probability scores $\mathbf{p}_{\theta}(x)$ (one for each output variable).
    \item a non-parametric and differentiable \textbf{loss}, which takes in $\mathbf{p}_{\theta}(x)$ and a label $\mathbf{y} \in \mathcal{Y}$ and produces a positive scalar that is minimized through gradient descent during training.
    \item a non-parametric \textbf{inference} module, which transforms $\mathbf{p}_{\theta}(x)$ into a predicted state $\hat{\mathbf{y}} \in \mathcal{Y}$.
\end{itemize}



In the context of informed classification, a neurosymbolic technique is a method to systematically integrate prior knowledge in a neural-based classification system. Most papers in the field assume that the architecture of the neural model (\eg fully connected, convolutional, transformer-based, etc.) mainly depends on the modality of the input space (\eg images, texts, etc.) and develop model-agnostic \textbf{neurosymbolic techniques} that integrate prior knowledge during learning, inference or both, but leave the design of the architecture outside the reach of the technique. We also consider neurosymbolic techniques for conformal classification \cite{ledaguenel2024conformal}, which integrate prior knowledge in the non-conformity measure and when computing the confidence set.


Recently, probabilistic neurosymbolic techniques, a family of techniques that leverage probabilistic reasoning to \textit{inform} a neural classifier with prior knowledge (see Figure \ref{fig:techniques}), have gained significant traction in the literature. We give below a short inventory of the main probabilistic neurosymbolic techniques found in the literature. Table \ref{tab:techniques} summarizes on which probabilistic reasoning problems each technique is built.

\begin{figure}[h]
\centering
\includegraphics[width=0.6\linewidth]{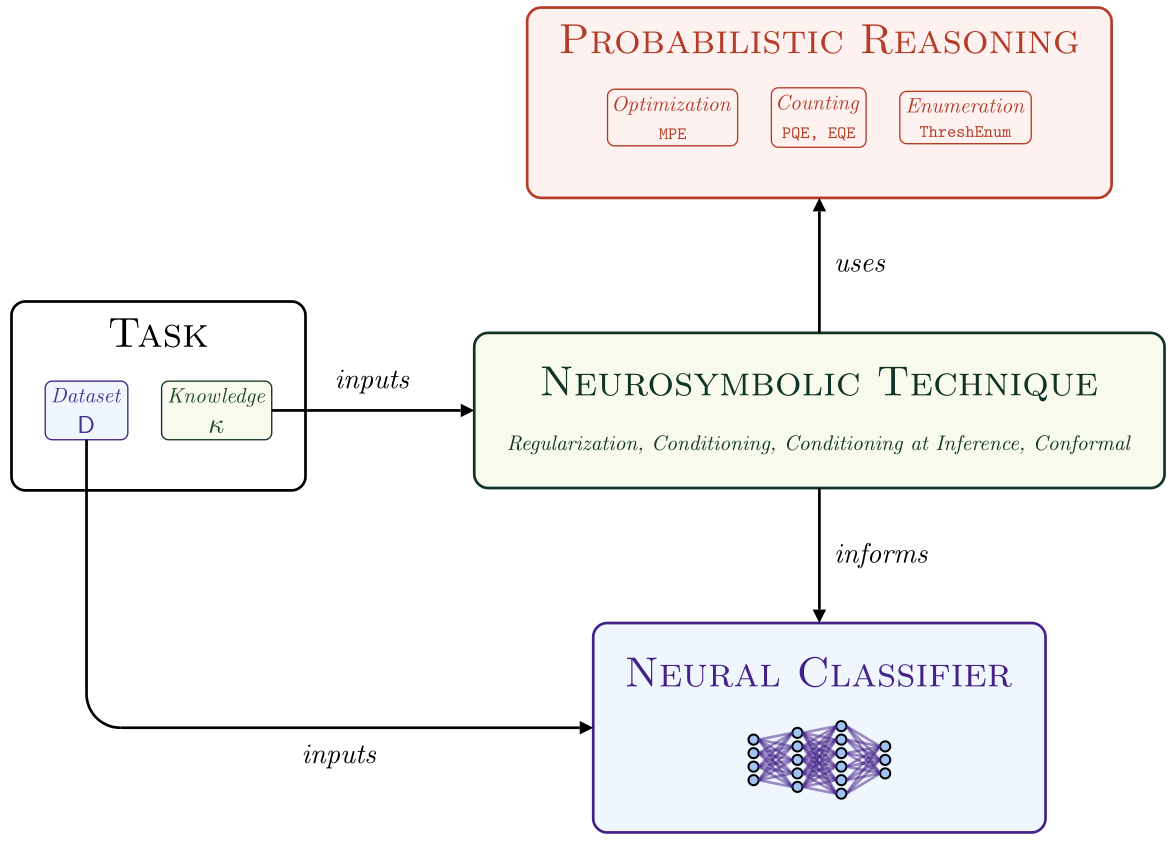}
\caption{A schematic illustration of probabilistic neurosymbolic techniques.}
\label{fig:techniques}
\end{figure}

\begin{table}[h]
    \centering
    \begin{tabular}{l c c c c}
        \hline
        \textbf{Techniques} & \texttt{MPE} & \texttt{ThreshEnum} & \texttt{PQE} & \texttt{EQE} \\
        \hline
        \multicolumn{4}{l}{\textbf{Classification}} \\
        Semantic loss \cite{Xu2018} & $\circ$ & $\circ$ & \cmark & $\circ$ \\
        Neurosymbolic entropy regularization \cite{Ahmed2022nesyer} & $\circ$ & $\circ$ & $\circ$ & \cmark \\
        Conditioning \cite{Deng2014,Ahmed2022spl,Yang2020} & \cmark & $\circ$ & \cmark & $\circ$ \\
        Conditioning at inference \cite{ledaguenel2024improving} & \cmark & $\circ$ & $\circ$ & $\circ$ \\
        \\
        \multicolumn{4}{l}{\textbf{Conformal classification}} \\
        Filtered conformal classification \cite{ledaguenel2024conformal} & $\circ$ & \cmark & $\circ$ & $\circ$ \\
        Conditioned conformal classification \cite{ledaguenel2024conformal} & $\circ$ & \cmark & \cmark & $\circ$ \\
        \hline
    \end{tabular}
    \caption{Summary of probabilistic neurosymbolic techniques: \cmark (resp. $\circ$) indicates that a technique relies (resp. does not rely) on solving a given probabilistic reasoning problem.}
    \label{tab:techniques}
\end{table}

\paragraph{Regularization} First introduced using fuzzy logics \cite{Diligenti2017a,Marra2019_tnorms,Badreddine2022}, regularization neurosymbolic techniques augment a standard multi-label loss (\eg binary cross-entropy) with an additional term that measures the consistency of logits with the prior knowledge, in order to steer the model towards \textit{valid} states. A probabilistic version, the semantic loss \cite{Xu2018}, uses the probability of prior knowledge $\mathcal{P}(\kappa|\mathbf{p}_{\theta}(x))$ (\texttt{PQE}) as a regularization term. Neurosymbolic entropy regularization \cite{Ahmed2022nesyer} is based on the principle of entropy regularization \cite{Grandvalet2004} but conditions the entropy term on prior knowledge $\kappa$. It relies on solving \texttt{EQE} during training. Regularization techniques are particularly adapted to the semi-supervised setting because the additional loss term is not based on any label.

\paragraph{Conditioning} Several neurosymbolic techniques build upon a probabilistic interpretation of neural networks: logits produced by the neural network are seen as parameters of a conditional probability distribution of the output given the input $\mathcal{P}(\cdot|\mathbf{p}_{\theta}(x))$, during training the loss computes the negative log-likelihood of the label under that distribution and during inference the most probable state given the learned distribution is predicted. A natural way to integrate prior knowledge $\kappa$ is to condition the distribution $\mathcal{P}(\cdot|\mathbf{p}_{\theta}(x))$ on $\kappa$. It was first introduced in \citep{Deng2014} for Hierarchical-Exclusion (HEX) graphs constraints. Semantic probabilistic layers \citep{Ahmed2022spl} can be used to implement semantic conditioning on tractable circuits. NeurASP \citep{Yang2020} defines semantic conditioning on a predicate extension of ASP programs. During training, computing the conditioned negative log-likelihood $- \log (\mathcal{P}(\mathbf{y} | \mathbf{p}_{\theta}(x), \kappa))$ relies on solving \texttt{PQE}. During inference, predicting the most probable output of $\mathcal{P}(\cdot | \mathbf{p}_{\theta}(x), \kappa)$ relies on solving \texttt{MPE}. An approached method for semantic conditioning on linear programs is proposed in \citep{niepert_implicit_2021}.

\paragraph{Conditioning at inference} Derived from semantic conditioning, semantic conditioning at inference \cite{ledaguenel2024improving} only applies conditioning during inference (\ie predicts the most probable state that satisfies prior knowledge) while retaining the standard negative log-likelihood loss during training. Thus, this technique only rely on solving \texttt{MPE} during inference.

\paragraph{Conformal classification techniques}
One of the great limitations of machine learning algorithms is their lack of guarantee regarding the validity of their predictions. Even when the algorithm is underpinned by a probabilistic interpretation, like often in deep learning, many experiments show that these probabilities are poorly calibrated. This results in a lack of trust in machine learning systems and is a major obstacle to their widespread adoption.

Conformal Prediction (CP) is a distribution-free and model agnostic framework that can solve this issue by transforming a machine learning algorithm from a point-wise predictor into a conformal predictor that outputs sets of predictions (called confidence sets) guaranteed to include the ground truth with a confidence level $1 - \alpha$, where $\alpha$ is a user-defined miscoverage rate. A good review on CP algorithms for multi-label classification tasks can be found in \cite{Wang2015}.

Probabilistic neurosymbolic techniques for informed conformal classification have recently been proposed in \cite{ledaguenel2024conformal}: instead of the \textbf{loss} and \textbf{inference} modules as mentioned above, prior knowledge is integrated in the \textbf{non-conformity measure} and when computing the \textbf{confidence set}. The first technique consists in filtering out invalid states (\ie states that do not satisfy prior knowledge $\kappa$) from the confidence set defined as $\{\mathbf{y} | \mathcal{P}(\mathbf{y}|\mathbf{p}_{\theta}(x)) \geq t\}$, where $t$ is a threshold determined by the calibration set. Therefore, this technique relies on solving \texttt{ThreshEnum}. A second technique further conditions the non-conformity measure on prior knowledge, therefore relying on both \texttt{PQE} and \texttt{ThreshEnum}.

\section{Scalability} \label{sec:scalability}
As mentioned in Section \ref{sec:techniques}, several neurosymbolic techniques rely on solving various probabilistic reasoning problems (\ie \texttt{MPE}, \texttt{ThreshEnum}, \texttt{PQE} or \texttt{EQE}). Hence, it is critical to understand the computational complexity of these problems to assess the scalability of such techniques to large classification tasks: ImageNet dataset \cite{Russakovsky2015} contains 1,000 classes, and up to 1,860 when adding parent classes in the WordNet hierarchy \cite{Miller1995}, Census Cars dataset \cite{Gebru2017} contains 2,675 of cars and iNaturalist dataset \cite{iNaturalist} contains 5,089 classes of natural species.

Most neurosymbolic techniques are defined using general purpose knowledge representation languages (\eg propositional logic, boolean circuits, linear programming, etc.). Although it makes sense to use a complete language to \textbf{define} a technique on the whole spectrum of knowledge, it also gives the false impression that the technique would \textbf{scale} properly on the whole spectrum of knowledge, which is never the case for two reasons: \textbf{tractability} and \textbf{succinctness}.

First, for most general purpose languages mentioned above, all probabilistic reasoning problems are \textbf{intractable}: there is no polynomial time algorithm (in the size of the input theory) that solves them (unless \texttt{P=NP}). Hopefully, there are complete fragments of these languages that provide \texttt{MPE}, \texttt{ThreshEnum}, \texttt{PQE} or \texttt{EQE}-tractability while remaining complete (see Section \ref{sec:kc}). Unfortunately, this is not enough, as these tractable languages tend to be less succinct.

Indeed, even though a complete language can represent any boolean function, it cannot do so \textbf{succinctly} (\ie with theories of size polynomial in the number of variables). This means that the size of the smallest theory to represent a boolean function is in general exponential in the number of variables. Since a reasoning algorithm must at least read the theory in its computations, it cannot scale on the whole spectrum of knowledge. This also means that all succinct languages are inherently specialized. Notice that graph-based languages are naturally succinct: the size of a graph is polynomially bounded by its number of edges or vertices, which are in bijection with the set of variables. Therefore, in this paper, we will often use graph-based languages to represent succinct fragments of complete languages.

These observations lead to the following definition of scalability.
\begin{definition}
    We say that a neurosymbolic technique is \textbf{scalable} on a given propositional language iff the language verifies two criteria:
\begin{itemize}
    \item \textbf{Succinctness}: theories in the language must be of polynomial size (in their number of variables).
    \item \textbf{Tractability}: the set of probabilistic reasoning problems on which the technique relies must be solvable in polynomial time in the size of the theories (for counting and optimization problems) or in the combined size of the theories and the output (for enumeration problems).
\end{itemize}
\end{definition}

\begin{remark}
    By definition, a neurosymbolic technique cannot be scalable on a complete language (which cannot be succinct), but only on specialized languages. Besides, we only look at time complexity for enumeration problems, but space complexity could become the main limiting factor for some algorithms.
\end{remark}



\section{Knowledge compilation} \label{sec:kc}
Knowledge compilation is the process of translating theories from a source language (\eg \texttt{CNF}) into a target language (\eg \texttt{d-DNNF}) that is tractable for a set of problems of interest to the user. We say that a source language $\mathtt{L}_s$ can be \textbf{compiled} into a target language $\mathtt{L}_t$ (noted $\mathtt{L}_s \to_c \mathtt{L}_t$) if any theory $\kappa_s \in \mathtt{T}_s$ has an equivalent theory $\kappa_t \in \mathtt{T}_s$ of polynomial size (\ie $|\kappa_t|=\mathcal{O}(p(|\kappa_s|))$ with $p$ a polynomial). Moreover, we say that $\mathtt{L}_s$ can be \textbf{efficiently compiled} into $\mathtt{L}_t$ (noted $\mathtt{L}_s \to_{ec} \mathtt{L}_t$) if there is a polynomial time algorithm that performs this translation. Notice that $\to_{c}$ and $\to_{ec}$ are both transitive relations and that $\mathtt{L}_s \to_{ec} \mathtt{L}_t$ implies $\mathtt{L}_s \to_c \mathtt{L}_t$. Efficient compilation from $\mathtt{L}_s$ to $\mathtt{L}_t$ is a key property for the study of computational complexity, as any problem tractable for $\mathtt{L}_t$ will also be tractable for $\mathtt{L}_s$.

A seminal work in the field was the compilation map introduced in \cite{Darwiche2002}. This map helps us understand the limits of several languages both in terms of succinctness (which languages can be compiled to which) and tractability (which reasoning problems can be performed in polynomial time for which languages).

In the context of probabilist reasoning, two fragments of boolean circuits are particularly relevant as target languages: \texttt{DNNF} and \texttt{d-DNNF}. \texttt{DNNF} is \texttt{MPE} and \texttt{ThreshEnum}-tractable \cite{Bourhis2022} and \texttt{d-DNNF} is additionally \texttt{PQE} and \texttt{EQE}-tractable \cite{Kiesel2023}. Therefore, if a language can be efficiently compiled into \texttt{d-DNNF} then it is \texttt{MPE}, \texttt{ThreshEnum}, \texttt{PQE} and \texttt{EQE}-tractable. Additionally, several other fragments of boolean circuits were identified as suitable target languages: Sentential Decision Diagrams (\texttt{SDD}) \cite{Darwiche2011} is a fragment of \texttt{d-DNNF} that offers polynomial negation, conjunction and disjunction. Besides, \cite{Darwiche2011} shows that a propositional formula $\kappa$ in conjunctive normal form with $k$ variables and a tree-width $\tau(\kappa)$ has an equivalent compressed and trimmed \texttt{SDD} of size $\mathcal{O}(k2^{\tau(\kappa)})$. Due to these properties, \texttt{SDD} has become a standard target logic for probabilistic neurosymbolic systems \citep{Xu2018,Ahmed2022spl}.

Knowledge compilation has both theoritical and practical interests. First, as mentioned above, it simplifies tractability proofs as several tractability results can be implied by a single efficient compilation property. From a practical standpoint, an algorithm for the source language can be obtained by composing the compilation algorithm with an algorithm for the target language. Besides, it enables to push as much of the computations during the \textit{offline} compilation phase to speed up the remaining computations needed in \textit{online} executing phase (see \textit{Compilation complexity} in Section \ref{sec:cc}). This is particularly interesting in the context of probabilistic neurosymbolic techniques, where prior knowledge is compiled once \textit{offline} in the required target language and is used \textit{online} many times (once for each learning step or test sample). Moreover, because algorithms for \texttt{MPE}, \texttt{PQE} and \texttt{EQE} based on \texttt{DNNF} and \texttt{d-DNNF} strictly follow the structure of the circuit during computations, they can be easily parallelized to leverage the power of GPUs. Finally, knowledge compilation offers a practical measure of complexity that goes beyond asymptotical complexity: for a given theory in a source language (\eg \texttt{CNF}), the size of the compiled theory determines the computing time of probabilistic reasoning problems for this specific theory, regardless of wether or not the source language was tractable.

Knowledge compilation has also some limitations, in particular in the context of probabilistic reasoning. First, the two main target languages identified so far for their tractability on probabilistic reasoning problems (\ie \texttt{DNNF} and \texttt{d-DNNF}) do not cover the full spectrum of tractability: we show this regarding \texttt{DNNF} and \texttt{MPE}-tractability with the case of matching constraints in Section \ref{sec:match}. Moreover, as most knowledge compilation algorithms developed so far directly compile into \texttt{d-DNNF} (or one of its fragments) \cite{Darwiche2004,Muise2012,Lagniez2017,Kiesel2023}, they cannot exploit the succinctness gap between \texttt{DNNF} and \texttt{d-DNNF}, and therefore neither the complexity gap between \texttt{MPE} and \texttt{PQE}. Finally, because algorithms follow the structure of the circuit during computations, this means that the computational graph cannot adapt to the particular probabilities used in the reasoning problem, as combinatorial solvers do to speed up computations. In practice, combinatorial solvers can therefore be faster than algorithms relying on knowledge compilation to boolean circuits.



\section{A complexity map} \label{sec:cmap}
In this section, we examine the tractability of several succinct languages. We use standard propositional languages defined in Section \ref{sec:kr} and the acronyms of Table \ref{tab:lang_acro} throughout the section. As mentioned earlier, a typical criteria to identify a tractable language is to show that it can compiled into a \texttt{CNF} of bounded tree-width. However, most types of prior knowledge commonly found in the neurosymbolic literature do not meet this criteria. Therefore, in this section, we analyze the tractability of succinct languages of unbounded tree-width.

Moreover, counting problems are known to be much harder in general than optimization problems \cite{Toda1991}. Therefore, it is natural to look for succinct languages which are \texttt{MPE}-tractable and for which \texttt{PQE} is still \#P-hard. These languages have great relevance in the context of probabilistic neurosymbolic techniques, as some techniques remain scalable on such languages (\eg semantic conditioning at inference) while others do not (\eg semantic conditioning and semantic regularization).

Our results regarding the tractability of succinct languages are displayed in the complexity map shown on Figure \ref{fig:cmap}. Then, Table \ref{tab:scalable} combine Figure \ref{fig:cmap} with Table \ref{tab:summary}, which summaries on which probabilistic reasoning problem each neurosymbolic technique relies, to determine which probabilistic neurosymbolic techniques are scalable depending on the type of tasks considered, characterized by their corresponding succinct representation languages.




\begin{figure}[h]
\centering
\includegraphics[width=\linewidth]{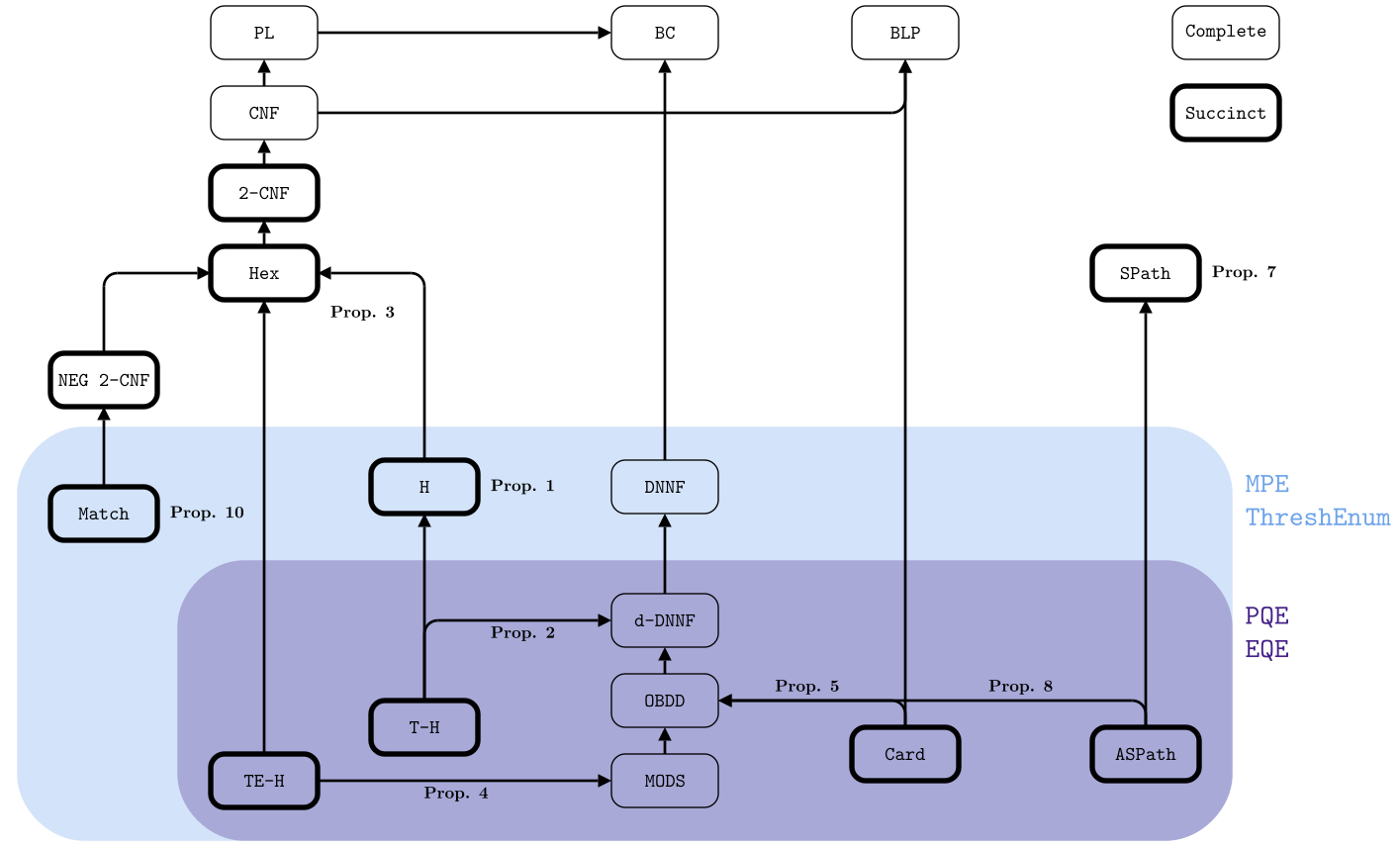}
\caption{A complexity map of probabilistic reasoning. An arrow $\mathtt{L_1} \to \mathtt{L_2}$ implies that $\mathtt{L_1}$ can be efficiently compiled to $\mathtt{L_2}$. Color regions indicate on which probabilistic reasoning problems a language is tractable: notice that the {\color{violet} tractability region of \texttt{PQE} and \texttt{EQE}} is included in the {\color{blue} tractability region of \texttt{MPE} and \texttt{ThreshEnum}}. Complete languages are represented with \textit{thin} frames and succinct languages with \textbf{thick} frames. When the tractability of a language $\mathtt{L_1}$ is proved through knowledge compilation $\mathtt{L_1} \to \mathtt{L_2}$, the corresponding proposition is referenced next to the arrow $\to$, otherwise the proposition is referenced next to the frame of the language $\mathtt{L_1}$.}
\label{fig:cmap}
\end{figure}

\begin{table}[h]
    \centering
    \begin{tabular}{l c c c c c c c c}
        \hline
        \textbf{Techniques} & \texttt{Hex} & \texttt{H} & \texttt{T-H} & \texttt{TE-H} & \texttt{Card} & \texttt{SP} & \texttt{ASP} & \texttt{Match} \\
        \hline
        \multicolumn{4}{l}{\textbf{Classification}} \\
        Semantic loss \cite{Xu2018} & \xmark & \xmark & \cmark & \cmark & \cmark & \xmark & \cmark & \xmark \\
        Neurosymbolic entropy regularization \cite{Ahmed2022nesyer} & \xmark & \xmark & \cmark & \cmark & \cmark & \xmark & \cmark & \xmark \\
        Conditioning \cite{Deng2014,Ahmed2022spl,Yang2020} & \xmark & \xmark & \cmark & \cmark & \cmark & \xmark & \cmark & \xmark \\
        Conditioning at inference \cite{ledaguenel2024improving} & \xmark & \cmark & \cmark & \cmark & \cmark & \xmark & \cmark & \cmark \\
        \\
        \multicolumn{4}{l}{\textbf{Conformal classification}} \\
        Filtered conformal classification \cite{ledaguenel2024conformal} & \xmark & \cmark & \cmark & \cmark & \cmark & \xmark & \cmark & \cmark \\
        Conditioned conformal classification \cite{ledaguenel2024conformal} & \xmark & \xmark & \cmark & \cmark & \cmark & \xmark & \cmark & \xmark \\
        \hline
    \end{tabular}
    \caption{Summary of probabilistic neurosymbolic techniques: \cmark (resp. \xmark) indicates that a technique is scalable (resp. does not scale) for a given succinct representation language. \texttt{SP} and \texttt{ASP} respectively stand for \texttt{SPath} and \texttt{ASPath}.}
    \label{tab:scalable}
\end{table}

\subsection{Hierarchical constraints} \label{sec:hex}
Hierarchical constraints are ubiquitous in artificial intelligence, because we are used to organize concepts in taxonomies, to the point where hierarchical classification (\ie a classification task where the set of output classes are organized in a hierarchy) is a field of research on its own.

These constraints are usually represented by a directed acyclic graph $G=(V, E_h)$ where the vertices in $V$ correspond to variables in $\mathbf{Y}$ and the edges $E_h$ express subsumption relations between those variables (\eg a dog is an animal). Therefore, such constraints naturally correspond to a directed acyclic and vertex-based language $\mathtt{H}:=(\mathtt{T_H}, \mathscr{s}_{\mathtt{H}})$ where a theory $(G:=(V, E), \varsigma) \in \mathtt{T_H}(\mathbf{Y})$ accepts a state $\mathbf{y} \in \mathbb{B}^{\mathbf{Y}}$ (\ie $\mathbf{y} \in \mathscr{s}_{\mathtt{H}}((G, \varsigma))$) if the vertices selected in $V$ respect the hierarchical constraints expressed in $G$: if a vertex $v \in V$ belongs to an accepted state $\mathbf{y}$ (\ie $\varsigma(v) \in \mathbf{y}$), then all its parents in $G$ also belong to $\mathbf{y}$. In other terms, the states accepted by a theory $(G:=(V, E), \varsigma)$ correspond exactly to closures in $G$.


\vspace{5mm}
\begin{proposition}
    \texttt{H} is \texttt{MPE} and \texttt{ThreshEnum}-tractable and \texttt{PQE} and \texttt{EQE}-intractable.
\end{proposition}

\begin{proof}
    \texttt{MPE} on \texttt{H} is equivalent to finding the maximum weighted closure in a weighted directed graph, which is can be done in time polynomial by reduction to a maximum flow problem (or equivalently minimum cut problem) \cite{Picard1976}. Likewise, this reduction can be used to solve \texttt{ThreshEnum} by using an algorithm for finding the $K$ best cuts (\ie the $k$ cuts of minimum weight) \cite{Hamacher1984}. Therefore \texttt{H} is \texttt{MPE} and \texttt{ThreshEnum}-tractable.

    Besides, \texttt{MC} on positive partitioned 2-\texttt{CNF} (\texttt{PP2CNF}) is known to be \#P-hard \cite{Provan1983}, where a formula $\kappa$ is \texttt{PP2CNF} if there is a bipartite graph $G=(U \cup V, E)$ such that $\kappa = \bigvee_{(u, v) \in E} (Y_u \lor Y_v)$. Notice that the number of models of $\kappa$ is equal to the number of models of $\kappa^*:= \bigvee_{(u, v) \in E} (Y_u \lor \neg Y_v)$. Hence, \texttt{MC} on \texttt{PP2CNF} can be reduced to \texttt{MC} on \texttt{H}. This implies that \texttt{H} is \texttt{MC}-intractable and therefore is \texttt{PQE} and \texttt{EQE}-intractable.
\end{proof}

If we impose that the directed acyclic graph $G$ is a tree, we obtain a fragment of the \texttt{H} that we will call \texttt{T-H}.
\vspace{5mm}
\begin{proposition}
    \texttt{T-H} is \texttt{MPE}, \texttt{ThreshEnum}, \texttt{PQE} and \texttt{EQE}-tractable.
\end{proposition}

\begin{proof}
    First, we show that \texttt{T-H} can be efficiently compiled to \texttt{2-CNF} of tree-width 1. A theory $(G:=(V, E), \varsigma) \in \mathtt{T_{T-H}}$ can be efficiently compiled to a \texttt{2-CNF} formula:
    \begin{equation} \label{eq:hex2cnf}
        \kappa_{(G, \varsigma)} := \bigg(\bigwedge_{(u, v) \in E} \varsigma(u) \lor \neg \varsigma(v) \bigg)
    \end{equation}
    Notice that the primal graph of $\kappa_{(G, \varsigma)}$ is $G$, therefore when the theory belongs to \texttt{T-H}, $G$ is a tree and its tree-width is 1.
    
    Since \texttt{T-H} can be efficiently compiled to 2-\texttt{CNF} of tree-width 1, it can also be compiled to \texttt{d-DNNF}. Therefore, \texttt{T-H} is \texttt{MPE}, \texttt{ThreshEnum}, \texttt{PQE} and \texttt{EQE}-tractable.
\end{proof}

This language can be enriched with exclusion edges that represent mutual exclusion between variables (\eg a state cannot be both a dog and a cat), like in HEX-graphs \cite{Deng2014}. In such language $\mathtt{Hex}:=(\mathtt{T}_{\mathtt{Hex}}, \mathscr{s}_{\mathtt{Hex}})$, a theory $(H=(V, E_h, E_e), \varsigma)$ is composed of a directed acyclic graph $(V, E_h)$ and an undirected graph $(V, E_e)$ sharing the same set of vertices $V$ which is in bijection with $\mathbf{Y}$ through $\varsigma$. Such a theory accepts a state $\mathbf{y} \in \mathbb{B}^{\mathbf{Y}}$ if both hierarchical constraints are satisfied as mentioned above and no two exclusive variables belong to $\mathbf{y}$: for two vertices such that $(u, v) \in E_e$, if one belong to $\mathbf{y}$ then the other does not. The \texttt{H} language can be seen as the fragment of the \texttt{Hex} language where the set of exclusion edges is empty.

\vspace{5mm}
\begin{remark}
    A \texttt{Hex} theory $(H=(V, E_h, E_e), \varsigma)$ can be efficiently compiled into a 2-\texttt{Horn} formula (a \texttt{CNF} formula where every clause contains 2 literals, with at most one positive literal):
    \begin{equation} \label{eq:hex2cnf}
        \kappa_H := \bigg(\bigwedge_{(u, v) \in E_h} \varsigma(u) \lor \neg \varsigma(v) \bigg) \land \bigg(\bigwedge_{(u, v) \in E_e} \neg \varsigma(u) \lor \neg \varsigma(v) \bigg)
    \end{equation}
\end{remark}
\vspace{5mm}
\begin{proposition}
    \texttt{Hex} is \texttt{MPE}, \texttt{ThreshEnum}, \texttt{PQE} and \texttt{EQE}-intractable.
\end{proposition}

\begin{proof}
  We know that \texttt{monotone 2-CNF} is \texttt{MPE}-intractable by reduction from the minimum vertex cover which is NP-hard \cite{Karp1972}. Therefore, it is also \texttt{ThreshEnum}-intractable by reduction from \texttt{MPE}. Likewise, we know that \texttt{MC} is intractable for monotone 2-\texttt{CNF} \cite{valiant1979} and can be reduced to \texttt{PQE} and \texttt{EQE}.
  
  Besides \texttt{MPE}, \texttt{ThreshEnum}, \texttt{PQE} and \texttt{EQE} on monotone 2-\texttt{CNF} can be reduced to their equivalent on negative 2-\texttt{CNF} (the fragment of \texttt{CNF} where clauses contain 2 negative literals), which can be efficiently compiled into a \texttt{Hex} theory using exclusion edges.
\end{proof}

Another interesting fragment of the \texttt{Hex} language, which we call \texttt{Exclusive-Hierarchical}, is composed of theories where a pair of variables is mutually exclusive (\ie there is an exclusion edge between the two variables) iff they have no common descendants (\ie no vertex can be reached from both variables following hierarchical edges). Notice that the exclusion edges are fully determined by the hierarchical edges. Hence, \texttt{Exclusive-Hierarchical} can be represented using the same syntax as \texttt{H} with a different semantic. We will use this representation in the rest of the section. The language \texttt{TreeExclusive-Hierarchical} is composed of \texttt{Exclusive-Hierarchical} theories $(G:=(V, E), \varsigma)$ where $G$ is a tree.
\vspace{5mm}
\begin{proposition}
    \texttt{TreeExclusive-Hierarchical} is \texttt{MPE}, \texttt{ThreshEnum}, \texttt{PQE} and \texttt{EQE}-tractable.
\end{proposition}

\begin{proof}
    This comes from the fact that satisfying states of a theory $(G:=(V, E), \varsigma)$ in \texttt{TreeExclusive-Hierarchical} can be enumerated in linear time: for each vertex $v \in V$, the state that only contains $v$ and its ancestors is accepted by $(G, \varsigma)$. The null state is the only other state accepted by $(G, \varsigma)$. Assume that a non-null state $\mathbf{y}$ is accepted by $(G, \varsigma)$ and note $v$ a vertex belonging to $\mathbf{y}$ such that no children of $v$ belong to $\mathbf{y}$. This vertex exists because $G$ is acyclic and $\mathbf{y}$ is non-null. Because $G$ is a tree, any vertex that is not an ancestor nor a descendant of $v$ does not share any descendants with $v$. Hence, $\mathbf{y}$ is the state that only contains $v$ and its ancestors: it contains the ancestors of $v$ by satisfaction of the hierarchical constraints, does not contain its descendants as per the choice of $v$ and cannot contain any other vertex by satisfaction of the exclusion constraints.
    
    \texttt{TreeExclusive-Hierarchical} can be efficiently compiled to \texttt{d-DNNF} simply by conjunction of the assignments corresponding to its accepted states. Therefore, \texttt{Exclusive-Hierarchical} is \texttt{MPE}, \texttt{ThreshEnum}, \texttt{PQE} and \texttt{EQE}-tractable.
\end{proof}

\subsection{Cardinal constraints} \label{sec:card}
Cardinal constraints operate on the number of variables included in a valid state. These constraints can be captured by the \texttt{Card} language, a fragment of \texttt{BLP} where theories consist in a single linear constraint $\langle \mathbf{Y}, \mathbf{1} \rangle \diamond l$ with $\diamond \in \{\leq, \geq, =\}$ and $0 \leq l \leq |\mathbf{Y}|$. An approximate algorithm for \texttt{PQE} based on the Gumbel trick and perturbed \texttt{MPE} calls was introduced in \cite{pogancic_differentiation_2019}. Then, \cite{AhmedICLR23} developed a polynomial time (and differentiable) algorithm to compute marginal quantities (including \texttt{PQE}) on k-subset constraints (\ie cardinal constraints reduced to $=$). We extend these tractabilty results to \texttt{ThreshEnum} and \texttt{EQE} through knowledge compilation.

\vspace{5mm}
\begin{remark}
    The \texttt{CARD} language defined in \cite{LeBerre2018} is more expressive than the \texttt{Card} fragment described above as it allows conjunction of cardinal constraints on literals and not only variables. This explains why \texttt{Card} can be compiled to \texttt{d-DNNF} (and even \texttt{OBDD}) while \texttt{CARD} cannot.
\end{remark}

\begin{proposition}\label{prop:card2dnnf}
    \texttt{Card} can be efficiently compiled to \texttt{d-DNNF}.
\end{proposition}

\begin{proof}
    A cardinal constraint can be efficiently compiled to an \texttt{d-DNNF} iteratively conditioning on variables following the natural order over the variables $\{Y_i\}_{1 \leq i \leq n}$ and keeping track of the number of positive variables at each step. See the complete proof in Section \ref{sec:proof_card}.
\end{proof}

\begin{corollary}
    \texttt{Card} is \texttt{MPE}, \texttt{ThreshEnum}, \texttt{PQE} and \texttt{EQE}-tractable.
\end{corollary}

\subsection{Simple paths constraints} \label{sec:paths}
Simple path constraints can be captured through the edge-based directed graphical language \texttt{SPath} where a state $\mathbf{y} \in \mathbb{B}^{\mathbf{Y}}$ satisfies a theory $(G:=(V, E), \varsigma)$ iff the set of selected edges (ie. $\{e \in E | \varsigma(e) \in \mathbf{y}\}$) forms a total simple path in the directed graph $G$.
\vspace{5mm}
\begin{example}
    An example of an edge-based directed theory $(G:=(V, E), \varsigma)$ is given on Figure \ref{fig:D}. It has one source $s$ and one sink $t$. The states that satisfy this theory regarding the total simple path semantics are $(Y_1, Y_4)$, $(Y_1, Y_3, Y_5, Y_6)$ and $(Y_2, Y_5, Y_6)$.
\end{example}

\begin{figure}[h]
    \centering
    \begin{tikzpicture}[>=Stealth, node distance=2cm]
      \node[circle,fill,inner sep=1.5pt, label=above:$s$] (1) {};
      \node[circle,fill,inner sep=1.5pt,right of=1] (2) {};
      \node[circle,fill,inner sep=1.5pt,below of=2] (3) {};
      \node[circle,fill,inner sep=1.5pt,right of=3] (4) {};
      \node[circle,fill,inner sep=1.5pt,right of=2, label=above:$t$] (5) {};
    
      \draw[->] (1) -- node[midway, above] {$Y_1$} (2);
      \draw[->] (1) -- node[midway, left] {$Y_2$} (3);
      \draw[->] (2) -- node[midway, right] {$Y_3$} (3);
      \draw[->] (2) -- node[midway, above] {$Y_4$} (5);
      \draw[->] (3) -- node[midway, below] {$Y_5$} (4);
      \draw[->] (4) -- node[midway, right] {$Y_6$} (5);
    \end{tikzpicture}
    \caption{An edge-based directed theory $(G:=(V, E), \varsigma)$: each edge $e$ is labeled with its corresponding variable $\varsigma(e) \in \mathbf{Y}$.}
    \label{fig:D}
\end{figure}
\vspace{5mm}

Simple path constraints are often encountered informed classification tasks (see Example \ref{ex:WSP1}) or other neurosymbolic AI domains (\eg reinforcement learning on routes \cite{Ling2021}).

\begin{example}[Warcraft Shortest Path] \label{ex:WSP1}
    The Warcraft Shortest Path dataset uses randomly generated images of terrain maps from the Warcraft II tileset. Maps are build on a $12 \times 12$ directed grid (each vertex is connected to all its \textit{neighbors}) and to each vertex of the grid corresponds a tile of the tileset. Each tile is a RGB image that depicts a specific terrain, which has a fixed traveling cost. For each map, the label encodes the shortest s-t path (\ie a path from the upper-left to the lower-right corners), where the weight of the path is the sum of the traveling costs of all terrains (\ie grid vertices) on the path. Terrain costs are used to produce the dataset but are not provided during training nor inference. This dataset has been used several times in the literature to build informed classification tasks to evaluate neurosymbolic techniques \cite{pogancic_differentiation_2019,Yang2020,niepert_implicit_2021,Ahmed2022spl}: prior knowledge about the task tells us that valid labels must correspond to simple paths in the grid.
\end{example}

Several work study probabilistic reasoning as well as knowledge compilation on simple path constraints in \textbf{undirected} graphs \cite{Nishino2017,Choi2017,Shen2019}. In this section we analyze the complexity of probabilistic reasoning for simple path constraints in \textbf{directed} graphs. In particular, we show that the acyclicity of the graph plays a crucial role in the tractability of probabilistic reasoning.

\vspace{5mm}
\begin{proposition}
    \texttt{SPath} is \texttt{MPE}, \texttt{ThreshEnum}, \texttt{PQE} and \texttt{EQE}-intractable.
\end{proposition}

\begin{proof}
    Solving \texttt{MPE} for a \texttt{SPath} theory $(G, \varsigma)$ is equivalent to finding a shortest path in $G$ with (positive and negative) real weights on the edges, which is known to be NP-hard by a polynomial reduction from the Hamiltonian path problem \cite{Karp1972}. This also implies that \texttt{SPath} is \texttt{ThreshEnum}-intractable. Likewise, counting the number of simple paths of a graph $G$ is known to be \#P-hard \cite{valiant1979} and has a trivial polynomial reduction to solving \texttt{PQE} and \texttt{EQE} for a \texttt{SPath} theory $(G, \varsigma)$.
\end{proof}


The fragment of \texttt{SPath} composed only of acyclic theories is called \texttt{ASPath}.

\begin{proposition}\label{prop:aspath2dnnf}
    \texttt{ASPath} can be efficiently compiled to \texttt{d-DNNF}.
\end{proposition}

\begin{proof}
    A simple path theory can be efficiently compiled to a \texttt{d-DNNF} by iteratively conditioning on variables following a topological ordering of the edges and keeping track of the last vertex reached by the path at each step. See the complete proof in Section \ref{sec:proof_path}.
\end{proof}

\begin{corollary}
    \texttt{ASPath} is \texttt{MPE}, \texttt{ThreshEnum}, \texttt{PQE} and \texttt{EQE}-tractable.
\end{corollary}


\begin{example}
    This case gives a great example how a better understanding of knowledge representation and probabilistic reasoning can influence the design of informed classification tasks. In its original version \cite{pogancic_differentiation_2019}, output variables correspond to vertices in the grid and a state satisfies the simple path constraint if the vertices set to $1$ constitute a simple s-t path. As pointed out in \cite{Ahmed2022spl}, the set of vertices ambiguously encode more than one path (because of cycles in the grid, there are several possible simple paths that go through the same vertices). Therefore, \cite{Ahmed2022spl} designs another version of the task where edges of the grid are chosen as output variables instead of vertices. However, because \texttt{SPath} is \texttt{MPE} and \texttt{PQE}-intractable, \cite{Ahmed2022spl} transforms prior knowledge to only keep simple paths with a maximal length of $29$ (the maximal length found in the training set) as valid states. This makes computations tractable but implies that test samples set might not be consistent with prior knowledge (it might contain a path longer than $29$ edges). Finally, such method would not scale to larger grids.
\end{example}



\subsection{Matching constraints} \label{sec:match}
Matching constraints naturally arise in artificial intelligence when one wants to find the best pairing between various entities (\eg individuals, tasks, resources, etc.). For instance, an informed classification task with matching constraints was build in \cite{Ahmed23} based on MNIST images \cite{LeCun1998}. Such constraints can be expressed using the edge-based undirected graphical language \texttt{Match} where a state $\mathbf{y} \in \mathbb{B}^{\mathbf{Y}}$ satisfies a theory $(G:=(V, E), \varsigma)$ iff the set of selected edges (ie. $\{e \in E | \varsigma(e) \in \mathbf{y}\}$) forms a \textbf{matching} in the graph $G$.
\vspace{5mm}
\begin{example}
    An example of an edge-based undirected theory $(G:=(V, E), \varsigma)$ is given on Figure \ref{fig:U}. The states that satisfy this theory regarding the perfect matching semantics are $(Y_2, Y_4, Y_7)$ and $(Y_1, Y_5, Y_8)$.
\end{example}

\begin{figure}[h]
    \centering
    \begin{tikzpicture}
      \node[circle, fill, inner sep=1.5pt] (A) at (0,0) {};
      \node[circle, fill, inner sep=1.5pt] (B) at (2,1) {};
      \node[circle, fill, inner sep=1.5pt] (C) at (3,-2) {};
      \node[circle, fill, inner sep=1.5pt] (D) at (1,-2) {};
      \node[circle, fill, inner sep=1.5pt] (E) at (-0.5,-1) {};
      \node[circle, fill, inner sep=1.5pt] (F) at (3.5,-0.5) {};
    
      \draw (A) -- node[midway, above] {$Y_1$} (B);
      \draw (A) -- node[midway, left] {$Y_2$} (E);
      \draw (B) -- node[midway, left] {$Y_3$} (C);
      \draw (C) -- node[midway, above] {$Y_4$} (D);
      \draw (D) -- node[midway, below] {$Y_5$} (E);
      \draw (A) -- node[midway, right] {$Y_6$} (D);
      \draw (B) -- node[midway, right] {$Y_7$} (F);
      \draw (C) -- node[midway, right] {$Y_8$} (F);
    \end{tikzpicture}
    \caption{An edge-based undirected theory $(G:=(V, E), \varsigma)$: each edge $e$ is labeled with its corresponding variable $\varsigma(e) \in \mathbf{Y}$.}
    \label{fig:U}
\end{figure}

\vspace{5mm}
\begin{proposition}
    \texttt{Match} is \texttt{MPE} and \texttt{ThreshEnum}-tractable but \texttt{PQE} and \texttt{EQE}-intractable.
\end{proposition}

\begin{proof}
    We build a polynomial algorithm for \texttt{MPE} on \texttt{Match} based on two facts. First, if a state is a matching, than any subset of that state is also a matching (no matching constraints can be violated by removing an edge from the matching). Secondly, we know that there is a polynomial time algorithm for finding a maximum weight-sum matching \cite{Edmonds1965}. Therefore, our algorithm for \texttt{MPE} on \texttt{Match} is the following: transform the probabilities to logits with the inverse sigmoid function $s^{-1}(p)=\log(\frac{p}{1-p)}$, remove the edges with negative weights from the graph and find a maximum weight-sum matching using Edmond's algorithm.
    
    The algorithm in \cite{Chegireddy1987} can be adapted to perform \texttt{RankedEnum} on matching constraints in polynomial time (in the combined size of the theory and the output). Interestingly, \texttt{ThreshEnum} can be reduced to \texttt{RankedEnum} (just enumerate the state in decreasing order of probability and stop at the first one that goes beyond the probability threshold), which means that \texttt{Match} is \texttt{ThreshEnum}-tractable.
    
    Finally, it is known that counting the number of matchings of a graph is \#-P-hard \cite{valiant1979}. Therefore, because \texttt{MC} can be reduced to both \texttt{PQE} and \texttt{EQE}, \texttt{Match} is \texttt{PQE} and \texttt{EQE}-intractable.
\end{proof}

\begin{proposition}[adapted from Theorem 8.3 in \cite{Amarilli2020}]
    For any monotone \texttt{CNF} $\kappa$ of bounded arity and degree, the size of the smallest \texttt{DNNF} equivalent to $\kappa$ is $2^{\Omega(\tau(\kappa))}$, where $\tau(\kappa)$ is the tree-width of $\kappa$.
\end{proposition}
\vspace{5mm}
\begin{remark}
    This also applies to negative \texttt{CNF} (\ie with only negative literals): otherwise we could simply rewrite a monotone \texttt{CNF} $\kappa^+$ of bounded arity and degree to a negative \texttt{CNF} $\kappa^-$ of bounded arity and degree (with variable change $Y_i \to \neg Z_i$), get a \texttt{DNNF} $C^-$ equivalent to $\kappa^-$ and smaller than $2^{\Omega(\tau(\kappa))}$, then negate all literals in $C^-$ to get a \texttt{DNNF} $C^+$ equivalent to $\kappa^+$.
\end{remark}

\begin{figure}[h]
    \hfill
    \begin{subfigure}[t]{0.4\linewidth}
        \centering
        \gridgraph{3}
        \caption{$3 \times 3$ grid}
        \label{fig:3x3grid}
    \end{subfigure}
    \hfill
    \begin{subfigure}[t]{0.4\linewidth}
        \centering
        \diamondgraph{3}
        \caption{Primal graph of $\kappa_{3 \times 3}$}
    \end{subfigure}
    \hfill \break
    
    \begin{subfigure}[t]{0.4\linewidth}
        \centering
        \gridgraph{5}
        \caption{$5 \times 5$ grid}
        \label{fig:5x5grid}
    \end{subfigure}
    \hfill
    \begin{subfigure}[t]{0.4\linewidth}
        \centering
        \diamondgraph{5}
        \caption{Primal graph of $\kappa_{5 \times 5}$}
    \end{subfigure}
\caption{$k \times k$ grid graphs and the corresponding primal graphs of $\kappa_{k \times k}$}
\label{fig:grids}
\end{figure}

\begin{proposition}\label{prop:match2dnnf}
    \texttt{Match} cannot be compiled to \texttt{DNNF}.
\end{proposition}

\begin{proof}
    It is easy to see that any matching theory $(G:=(V, E), \varsigma)$ can compiled into a negative 2-\texttt{CNF} formula (a \texttt{CNF} formula with at most two literals per clause):
    \begin{equation}
        \kappa_G := \bigwedge_{v \in V} \bigwedge_{\substack{e_i=(u, v) \in E, \\ e_j=(w, v) \in E, \\ i \ne j}}  (\neg Y_i \lor \neg Y_j)
    \end{equation}

    The degree of $\kappa_G$ corresponds to the degree of $G$, hence if $G$ is of bounded degree then $\kappa_G$ is too. Besides, we know that $k \times k$ grid graphs have a bounded degree of $4$ and a tree-width of $k$ \cite{Diestel2017}. We can also show that the \texttt{CNF} $\kappa_{k \times k}$ corresponding to matching constraints on the $k \times k$ grid has a tree-width in $\mathcal{O}(k)$. In fact, it is easy to notice that for odd values of $k$, the $k \times k$ grid is a subgraph of the primal graph of $\kappa_{k \times k}$. We give two examples for $k=3$ and $k=5$ on Figure \ref{fig:grids}. Since $\kappa_{k \times k}$ has a bounded arity (2) and a bounded degree (4) but unbounded tree-width in $\mathcal{O}(k)$, the smallest \texttt{DNNF} equivalent to $\kappa_{k \times k}$ has an exponential size. Therefore, matching constraints on $k \times k$ grid graphs can not be compiled to \texttt{DNNF}.

    To conclude, \texttt{Match} cannot be compiled to \texttt{DNNF}.
\end{proof}


    

\section{Proofs} \label{sec:proofs}
\subsection{Cardinal constraints} \label{sec:proof_card}
In this section, we prove Proposition \ref{prop:card2dnnf}:
\begin{proposition*}
    \texttt{Card} can be efficiently compiled to \texttt{d-DNNF}.
\end{proposition*}

\begin{figure}[htb]
    \hfill
    \begin{subfigure}[t]{0.3\linewidth}
        \includegraphics[width=\textwidth]{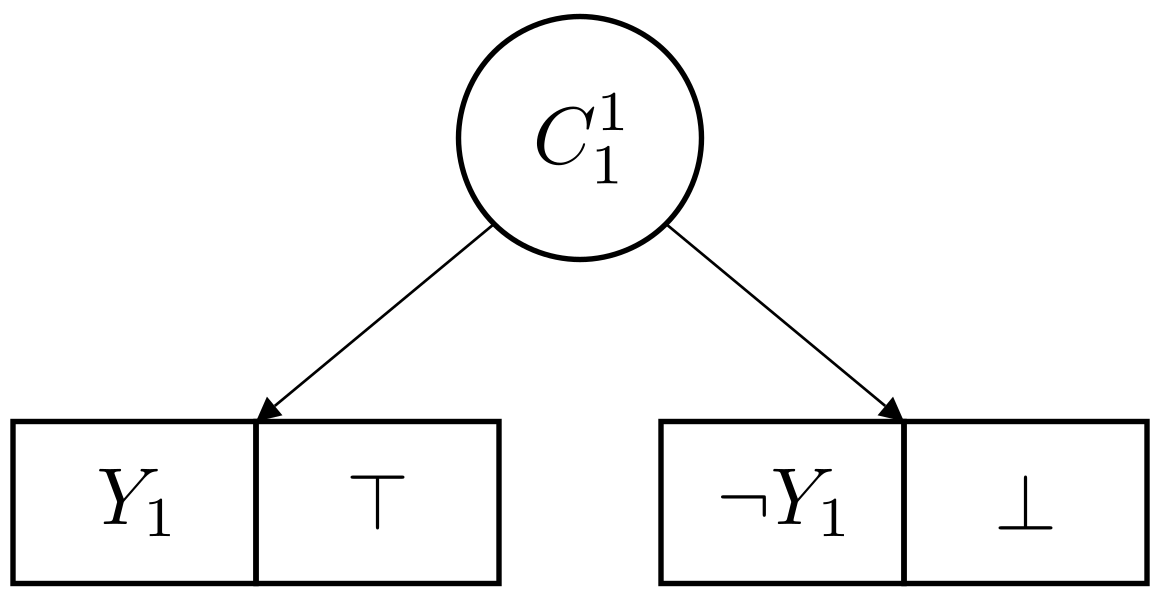}
        \caption{$C_1^1$}
        \label{fig:link_C11}
    \end{subfigure}
    \hfill
    \begin{subfigure}[t]{0.3\linewidth}
        \includegraphics[width=\textwidth]{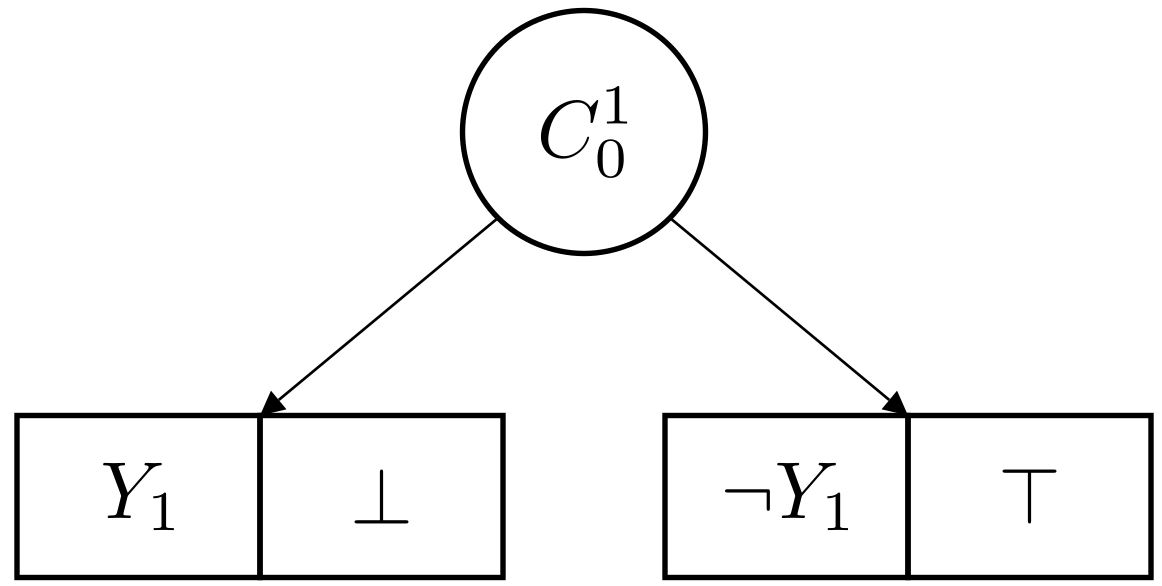}
        \caption{$C_0^1$}
        \label{fig:link_C01}
    \end{subfigure}
    \hfill \break
    
    \begin{subfigure}[t]{0.3\linewidth}
        \includegraphics[width=\textwidth]{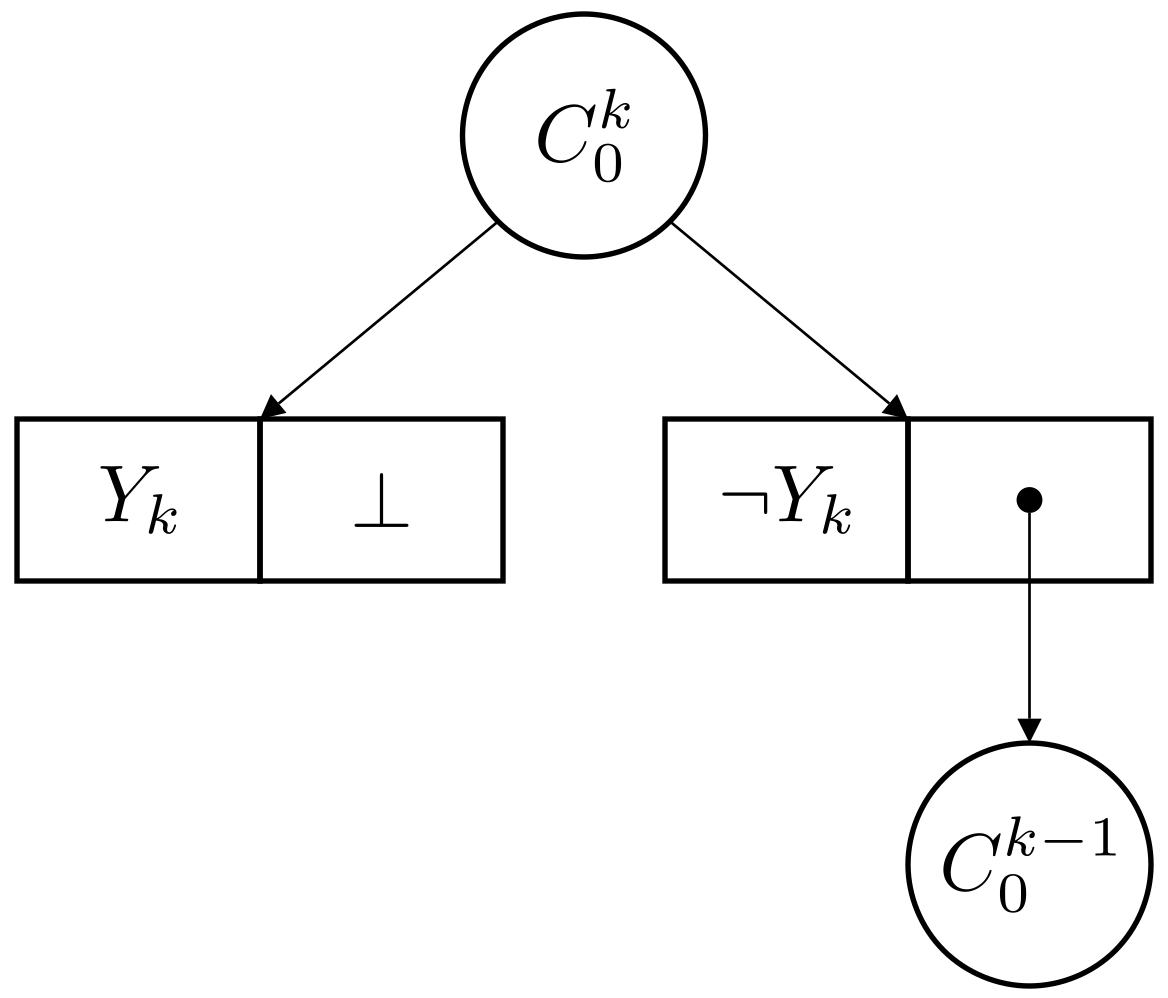}
        \caption{$C_0^k$ for $k>1$}
        \label{fig:link_C0k}
    \end{subfigure}
    \hfill
    \begin{subfigure}[t]{0.3\linewidth}
        \includegraphics[width=\textwidth]{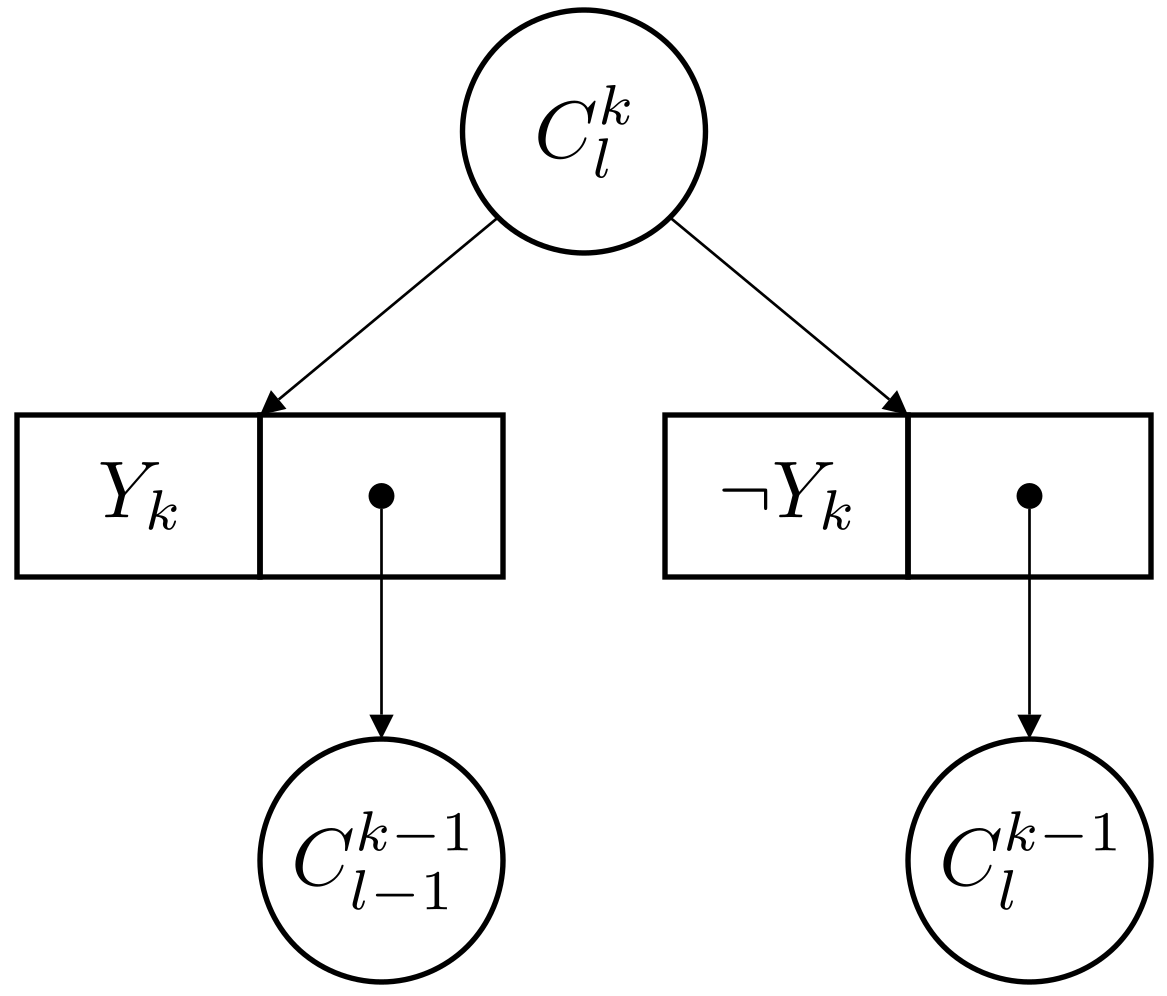}
        \caption{$C_l^k$ for $0<l<k$}
        \label{fig:link_Clk}
    \end{subfigure}
    \hfill
    \begin{subfigure}[t]{0.3\linewidth}
        \includegraphics[width=\textwidth]{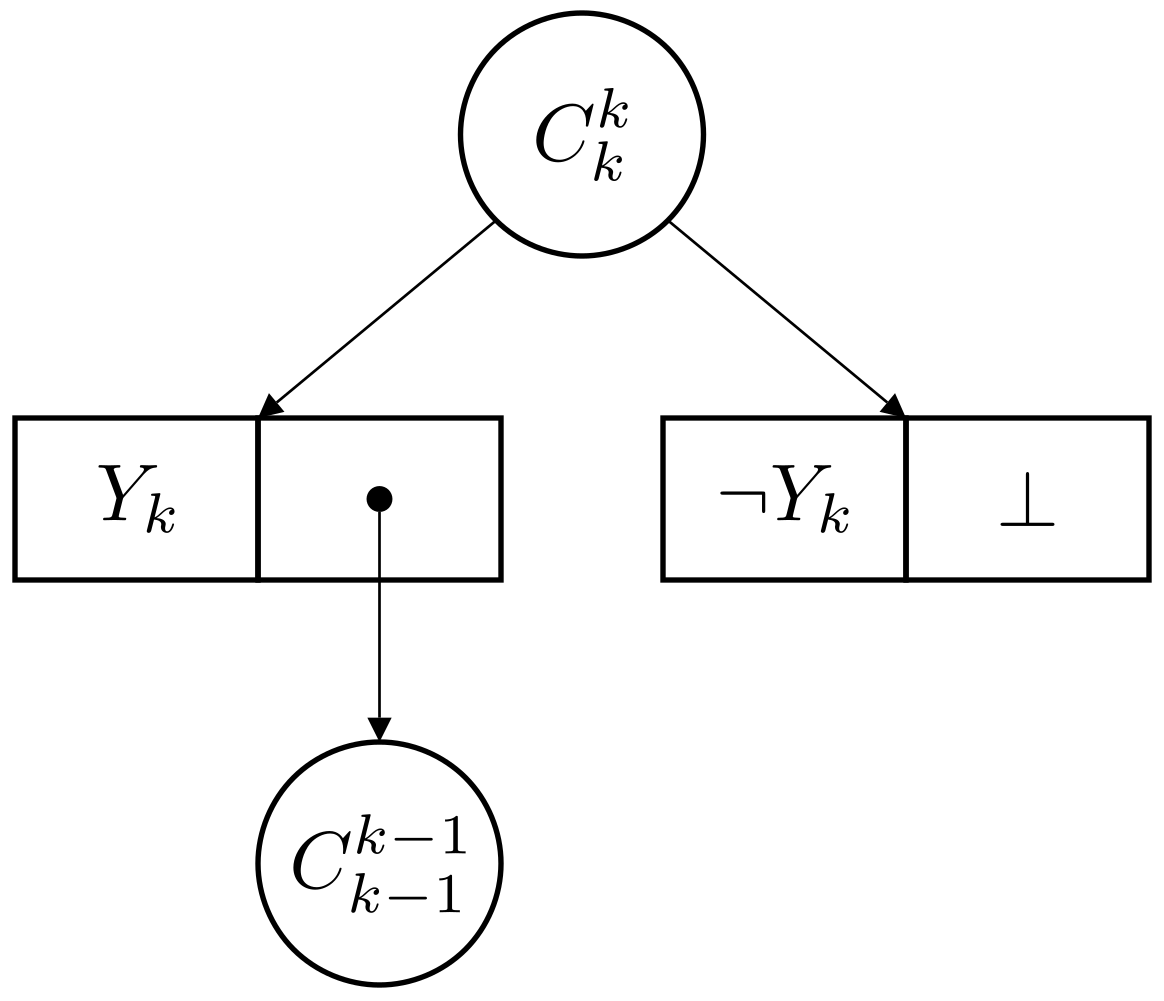}
        \caption{$C_k^k$ for $k>1$}
        \label{fig:link_Clk}
    \end{subfigure}

\caption{Decision nodes $C_l^k$}
\label{fig:link_nodes}
\end{figure}

\begin{proof}
    We prove below that Figure \ref{fig:link_nodes} gives a template to translate a cardinal constraint $r := \langle \mathbf{Y}_{1:k}, \mathbf{1} \rangle=l$ with $k \geq 1$ into an equivalent \texttt{d-DNNF} $C_l^k$ of polynomial size. It also gives us an efficient compilation algorithm to perform this translation in polynomial time: create all decision nodes and connect them appropriately.

    First, notice that only variables in $C_l^k$ are negated, therefore $C_l^k$ is in \texttt{NNF}. Moreover, $C_l^k$ is only composed of decision nodes on $\{Y_j\}_{1 \leq j \leq k}$. This implies that each $\land$-node is \textbf{decomposable}: the left side contains a variable $Y_m$ while the right side contains either no variables or variables $\{Y_j\}_{1 \leq j \leq m-1}$. This also implies that each $\lor$-node is \textbf{deterministic}: one side accepts a variable $Y_m$ while the other accepts $\neg Y_m$, hence they cannot be satisfied jointly. Therefore, $C_l^k$ is a \texttt{d-DNNF}.
    
    Besides, $C_l^k$ is only composed of nodes $(C_i^j)_{1 \leq j \leq k, 0 \leq i \leq \min(j, l)}$ with $6$ wires each, meaning that the size of $C_l^k$ is in $\mathcal{O}(k^2)$.

    Finally, we show by recurrence on $k$ that the circuit $C_l^k$ with $k \geq 1$ and $0 \leq l \leq k$ accepts a state $\mathbf{y} \in \mathbb{B}^{\mathbf{Y}_{i:k}}$ iff it contains exactly $l$ variables (ie. $|\mathbf{y}|=l$):
    \begin{itemize}
        \item Initialization for $k=1$: $C_1^1$ only accepts $y_1=1$ and $C_0^1$ only accepts $y_1=0$.
        \item Heredity from $k$ to $k+1$, for $0 \leq l \leq k$ and $\mathbf{y} \in \mathbb{B}^{\mathbf{Y}_{i:k+1}}$:
        \begin{itemize}
            \item if $l=0$: $C_0^{k+1}(\mathbf{y})=1$ iff $y_{k+1}=0$ and $C_0^k(\mathbf{y}_{i:k})=1$, which means $C_0^{k+1}$ accepts $\mathbf{y}$ iff:
            $$|\mathbf{y}|=|\mathbf{y}_{i:k}|+y_{k+1}=0+0=0$$
            \item if $l>0$: $C_l^{k+1}(\mathbf{y})=1$ in only two cases:
            \begin{itemize}
                \item if $y_{k+1}=0$ and $C_l^k(\mathbf{y}_{i:k})=1$, which means we have: $$|\mathbf{y}|=|\mathbf{y}_{i:k}|+y_{k+1}=l+0=l$$
                \item if $y_{k+1}=1$ and $C_{l-1}^k(\mathbf{y}_{i:k})=1$, which means we have: $$|\mathbf{y}|=|\mathbf{y}_{i:k}|+y_{k+1}=l-1+1=l$$
            \end{itemize}
        \end{itemize}
    \end{itemize}
\end{proof}

\begin{remark}
    This algorithm actually compiles cardinal constraints into an \texttt{OBDD} as all $\lor$-nodes are decision nodes ordered by the natural order on $\{Y_j\}_{1 \leq j \leq k}$. Besides, the algorithm can be easily modified to efficiently compile a cardinal constraint $\langle \mathbf{Y}_{1:k}, \mathbf{1} \rangle \leq l$ into an \texttt{OBDD} by replacing $C^1_1$ and $C^k_k$ nodes by $\top$ nodes. Similarly, cardinal constraints $\langle \mathbf{Y}_{1:k}, \mathbf{1} \rangle \geq l$ can be efficiently compiled to \texttt{OBDD}.
\end{remark}

\subsection{Simple path constraints} \label{sec:proof_path}
In this section, we prove Proposition \ref{prop:aspath2dnnf}:
\begin{proposition*}
    \texttt{ASPath} can be efficiently compiled to \texttt{d-DNNF}.
\end{proposition*}

\vspace{5mm}
\begin{figure}[b]
\begin{subfigure}[t]{0.30\linewidth}
    \centering
    \includegraphics[width=\textwidth]{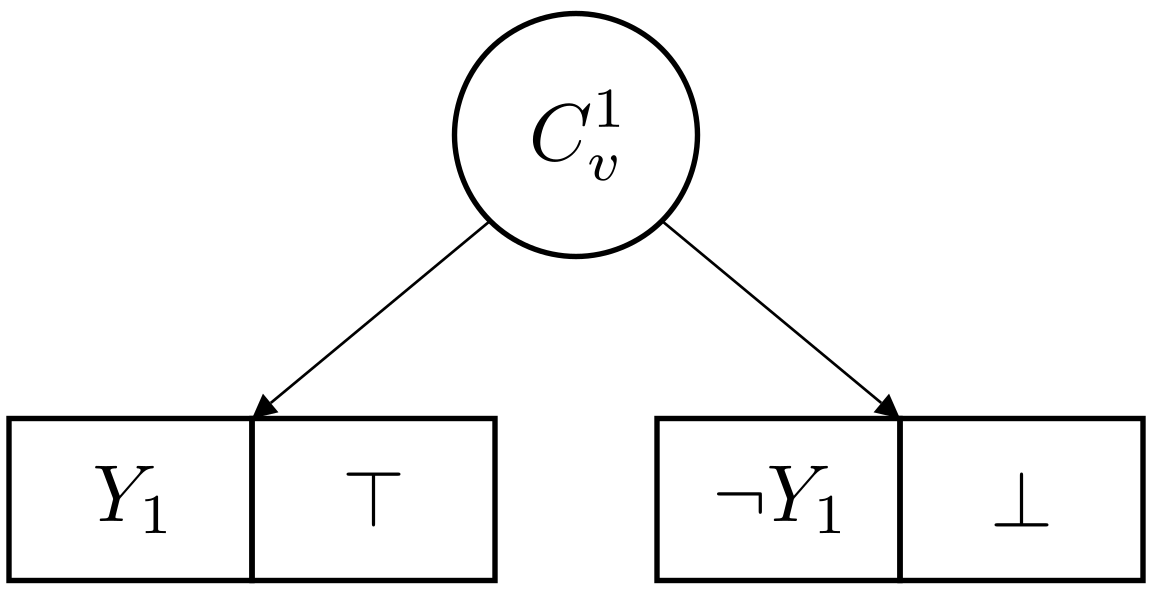}
    \caption{{\scriptsize $C_v^1$ if $e_1=(s, v)$}}
    \label{fig:Cnk_outgoing}
\end{subfigure}
\hfill
\begin{subfigure}[t]{0.30\linewidth}
    \centering
    \includegraphics[width=\textwidth]{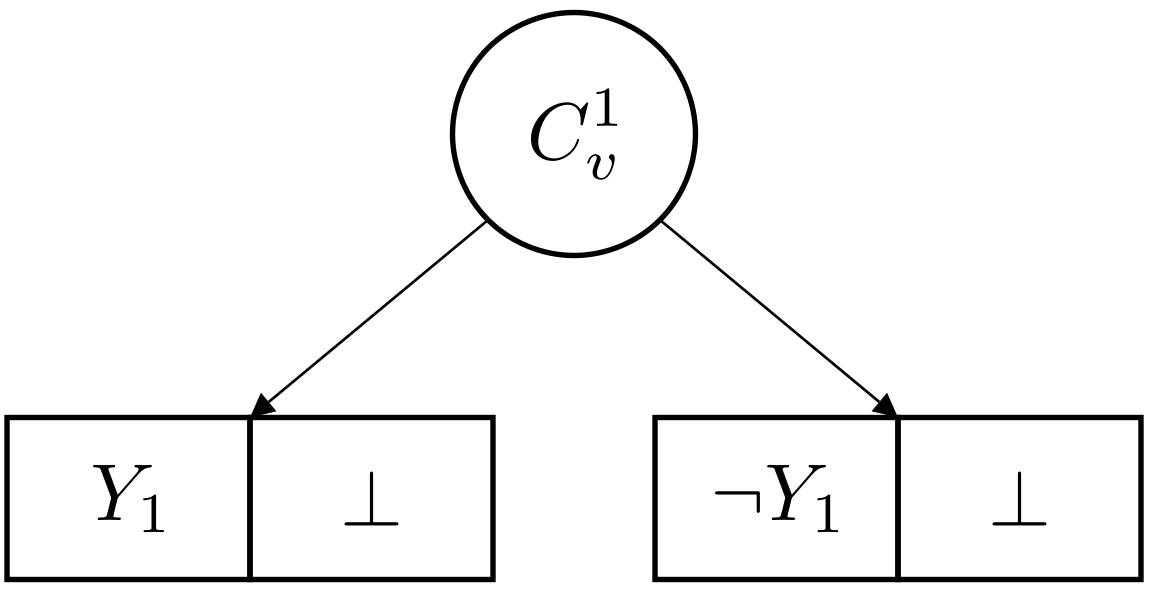}
    \caption{{\scriptsize $C_v^1$ if $v \neq s$ and $e_1 \neq (s, v)$}}
    \label{fig:Cnk}
\end{subfigure}
\hfill
\begin{subfigure}[t]{0.30\linewidth}
    \centering
    \includegraphics[width=\textwidth]{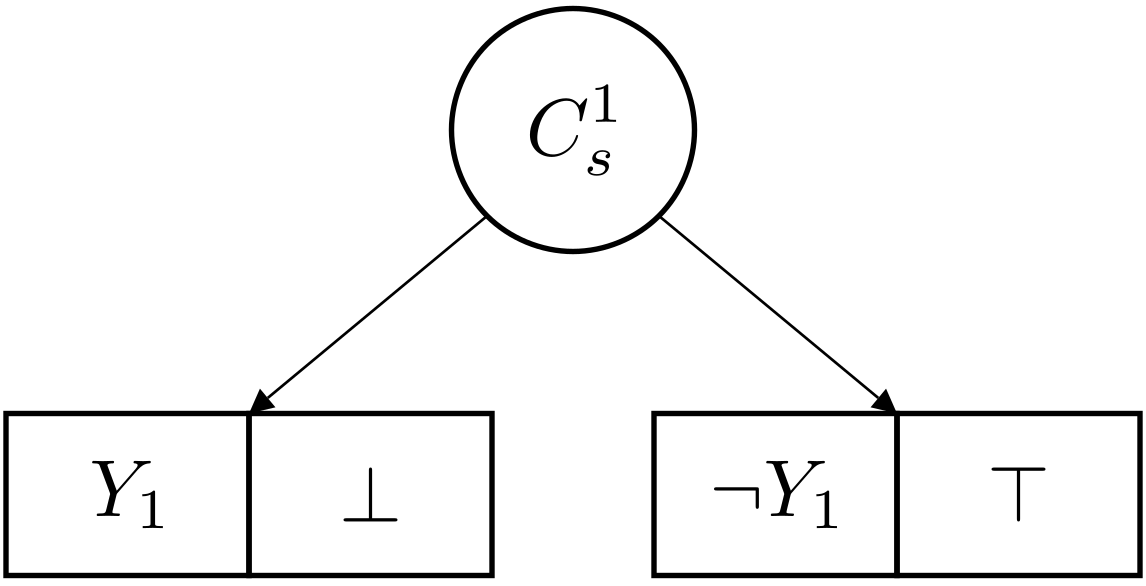}
    \caption{{\scriptsize $C_s^1$}}
    \label{fig:Cs1}
\end{subfigure}
\caption{Initial decision nodes $C_v^1$ for $v \in V$}
\label{fig:spk_nodes}
\end{figure}

\begin{figure}[t]
\hspace*{\fill}
\begin{subfigure}[t]{0.30\linewidth}
    \centering
    \includegraphics[width=\textwidth]{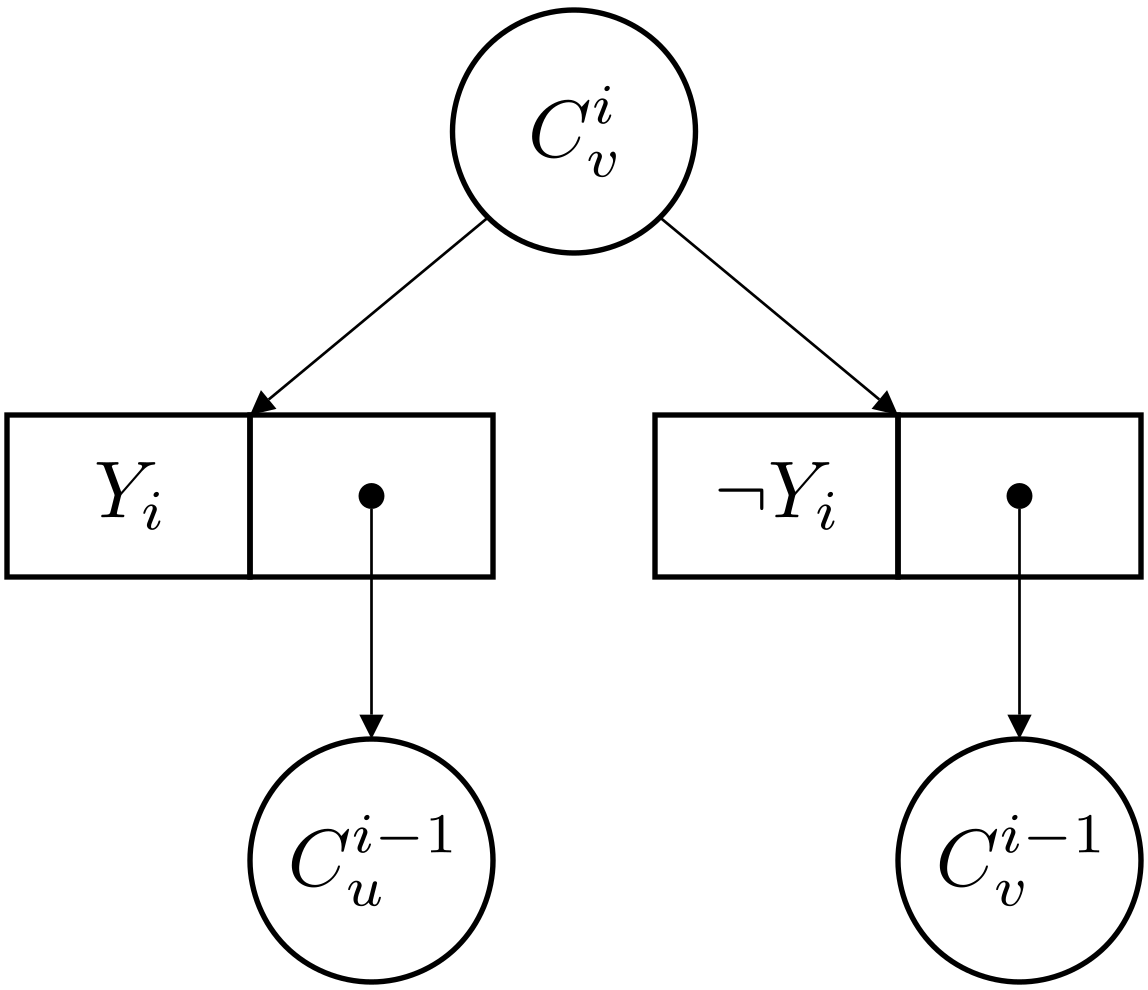}
    \caption{{\scriptsize $C_v^{i+1}$ if $e_{i+1}=(u, v)$}}
    \label{fig:Cni_outcoming}
\end{subfigure}
\hfill
\begin{subfigure}[t]{0.30\linewidth}
    \centering
    \includegraphics[width=\textwidth]{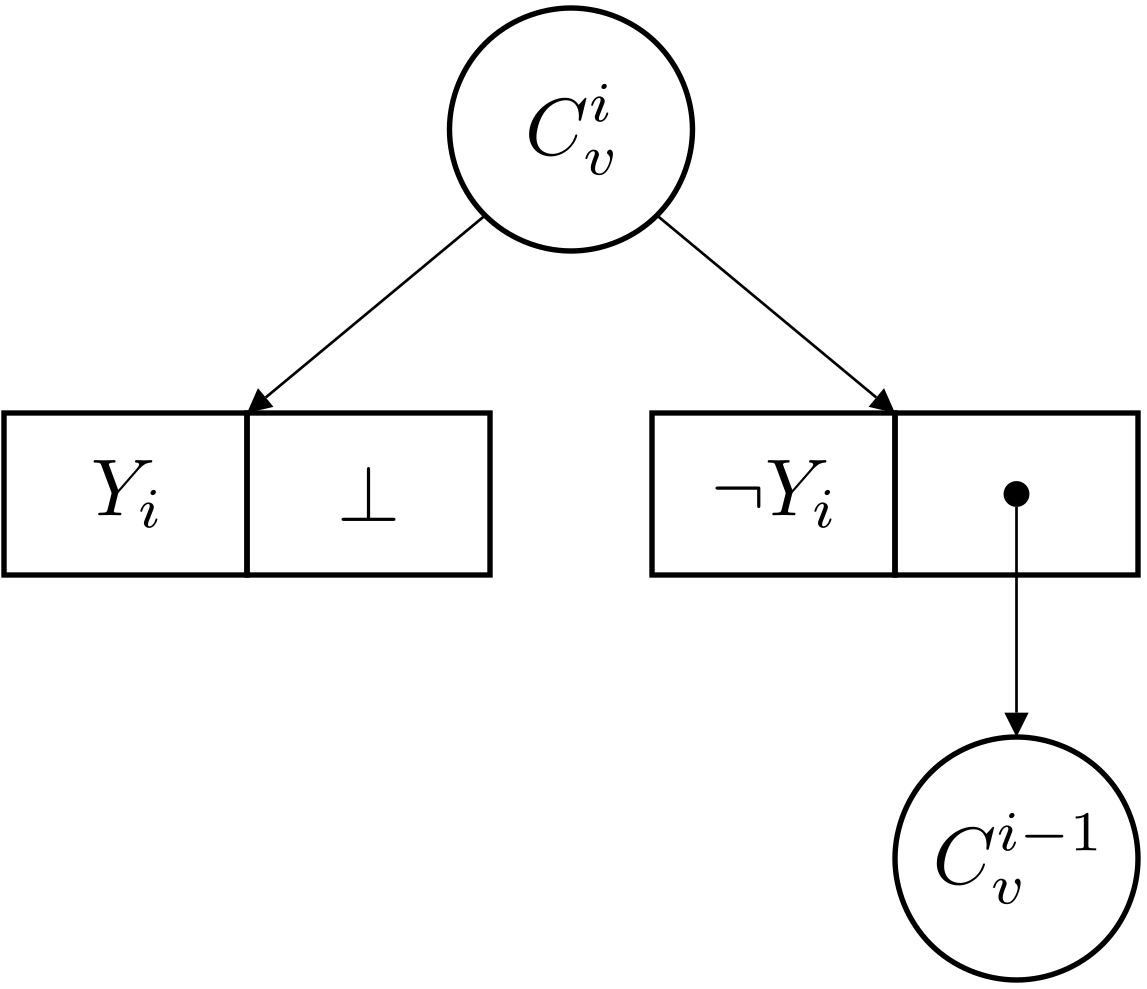}
    \caption{{\scriptsize $C_v^{i+1}$ if $e_{i+1}$ is not an incoming edge of $v$}}
    \label{fig:Cni_not_outcoming}
\end{subfigure}
\hspace*{\fill}
\caption{Decision nodes $C_v^i$ for $1 < i \leq k$ and $v \in V$}
\label{fig:spi_nodes}
\end{figure}

\begin{proof}
    Let's assume an acyclic simple path theory $(D:=(V, E), \varsigma) \in \mathtt{T_{ASP}}(\mathbf{Y})$ where $\mathbf{Y}:=\{Y_i\}_{1 \leq i \leq |E|}$. For simplicity, we note $e_i$ the edge such that $\varsigma(e_i) = Y_i$. We assume without loss of generality that $D$ only contains a single source and sink and that $\varsigma$ corresponds to a topological ordering of the edges in $E$ (\ie for $e_i=(u_i, v_i), e_j=(u_j, v_j) \in E$ there is a path from $v_i$ to $u_j$ iff $i<j$). If this is not the case:
    \begin{itemize}
        \item chose one source vertex $s$, delete all the others and reconnect their outgoing edges to $s$
        \item chose a sink vertex $t$, delete all the others and reconnect their incoming edges to $t$
        \item change $\varsigma$ to labeling that corresponds to a topological order, this can be reverted after compilation by renaming the variables to get a circuit equivalent to $(D:=(V, E), \varsigma)$
    \end{itemize}

    For a vertex $u \in V$, we will note $\varsigma_m(u)$ the index of the first incident edge to $u$ and $\varsigma_M(u)$ the index of the last outgoing edge of $u$. We also note $D^j$ the graph that contains edges $(e_i)_{1 \leq i \leq j}$ and all vertices that are endpoints of those edges.

    We prove below that Figures \ref{fig:spk_nodes} and \ref{fig:spi_nodes} gives a template to translate $(D:=(V, E), \varsigma)$ with $k \geq 1$ into an equivalent \texttt{d-DNNF} $C_t^k$ of polynomial size. It also gives us an efficient compilation algorithm to perform this translation in polynomial time: create all decision nodes and connect them appropriately.

    Similar to the proof of Proposition \ref{prop:card2dnnf}, we show that $C_t^k$ is a \texttt{d-DNNF}. First, only variables in $C_t^k$ are negated, therefore $C_t^k$ is in \texttt{NNF}. Moreover, $C_t^k$ is only composed of decision nodes on $\{Y_j\}_{1 \leq j \leq k}$. This implies that each $\land$-node is \textbf{decomposable}: the left side contains a variable $Y_m$ while the right side contains either no variables or variables $\{Y_j\}_{1 \leq j \leq m-1}$. This also implies that each $\lor$-node is \textbf{deterministic}: one side accepts a variable $Y_m$ while the other accepts $\neg Y_m$, hence they cannot be satisfied jointly. Therefore, $C_t^k$ is a \texttt{d-DNNF}.

    Besides, $C_t^k$ has at most $|V| \times k = \mathcal{O}(k^3)$ decision nodes with $6$ wires each, meaning that the size of the circuit is in $\mathcal{O}(k^3)$.

    Notice that $C_s^i$ with $1 \leq i \leq k$ only accepts the null state $\mathbf{0} \in \mathbb{B}^{\mathbf{Y}_{1:i}}$. We prove this by recurrence on $i$:
    \begin{itemize}
        \item Initialization for $i=1$: $C_s^{1}(y_1)=1$ iff $y_1=0$ by definition (see Figure \ref{fig:Cs1}).
        \item Heredity from $i$ to $i+1$: $s$ has no incoming edge in $D$ (because it is a source vertex), hence $C_s^{i+1}$ accepts $\mathbf{y} \in \mathbb{B}^{\mathbf{Y}_{1:i+1}}$ iff $y_{i+1}=0$ and $C_s^{i}(\mathbf{y}_{1:i})=1$. Which means by the recurrence hypothesis that $\mathbf{y}_{1:i}=\mathbf{0}$ and therefore $\mathbf{y}=0$.
    \end{itemize}

    Then, let's show that if $\mathbf{y} \in \mathbb{B}^{\mathbf{Y}_{1:i}}$ represents a total simple path $s \to v$ in $D^i$, $y_i=1$ and $e_i=(u, w)$, then $w=v$ and $u \in D^{i-1}$.
    
    We first show that $u \in D^{i-1}$:
    \begin{itemize}
        \item if $u=s$, then $u \in D^1 \subset D^{i-1}$.
        \item if $u \ne s$, since $\mathbf{y}$ represents a path $s \to v$, there is an edge $e_j=(r, u)$ with $j < i$, which means that $u \in D^{i-1}$. 
    \end{itemize}
    
    Now let's show that $w=v$ reasoning by the absurd. Let's assume that $w \ne v$, then since $\mathbf{y}$ represents a path from $s$ to $v$, there is an edge $e_l=(q, v)$ with $l<i$. Hence, since $y_i=1$, $e_i$ is on the path from $s$ to $v$ which implies that there is a path from the end point $w$ of $e_i$ to the start point $q$ of $e_l$ with $l<i$, which is in contradiction with edges following a topological order.

    We can now show that the circuit $C_v^i$ with $1 \leq i \leq k$ and $v \in D^i \setminus s$ accepts a state $\mathbf{y} \in \mathbb{B}^{\mathbf{Y}_{1:i}}$ iff it represents a total simple path in $D^i$. In particular, this shows that $C_t^k$ is equivalent to $(D, \varsigma)$ and concludes the proof.

    We will proceed by recurrence on $i$, first showing that all accepted states by $C_v^i$ are paths $s \to v$ in $D^i$ then showing that only them are accepted.
    \begin{itemize}
        \item Initialization for $i=1$: $D^1$ only contains the vertices $s$ and $v$ such that $e_1=(s, v)$ and $C_v^1$ accepts exactly $y_1=1$ which is the only path $s \to v$ in $D^1$.
        \item Heredity from $i$ to $i+1$:
        \begin{itemize}
            \item Assume a state $\mathbf{y} \in \mathbb{B}^{\mathbf{Y}_{1:i+1}}$ represents a path $s \to v$ in $D^{i+1}$ and note:
            \begin{itemize}
                \item if $y_{i+1}=1$, then according to Lemma \ref{lem:topo} $e_{i+1}=(u, v)$ for some $u \in D^i$
                \begin{itemize}
                    \item if $u=s$: then $\mathbf{y}_{1:i}=\mathbf{0}$ and by Lemma \ref{lem:Csi} $C_v^{i+1}(\mathbf{y})=C_s^{i}(\mathbf{y}_{1:i})=1$
                    \item if $u \ne s$: then $\mathbf{y}_{1:i}$ represents a path $s \to u$ in $D^i$ and by the recurrence hypothesis $C_v^{i+1}(\mathbf{y})=C_u^i(\mathbf{y}_{1:i})=1$
                \end{itemize}
                \item if $y_{i+1}=0$, then $\mathbf{y}_{1:i}$ represents a path $s \to v$ in $D^i$ and by the recurrence hypothesis $C_v^{i+1}(\mathbf{y})=C_v^i(\mathbf{y}_{1:i})=1$.
            \end{itemize}
            \item Assume a state $\mathbf{y} \in \mathbb{B}^{\mathbf{Y}_{1:i+1}}$ is accepted by $C_v^{i+1}$:
            \begin{itemize}
                \item if $y_{i+1}=1$, then according to Lemma \ref{lem:topo} $e_{i+1}=(u, v)$ for some $u \in D^i$
                \begin{itemize}
                    \item if $u=s$: then $C_s^{i}(\mathbf{y}_{1:i})=C_v^{i+1}(\mathbf{y})=1$, hence by Lemma \ref{lem:Csi} $\mathbf{y}_{1:i}=\mathbf{0}$ and $\mathbf{y}$ represents a path $s \to v$ in $D^{i+1}$.
                    \item if $u \ne s$: then $C_u^i(\mathbf{y}_{1:i})=C_v^{i+1}(\mathbf{y})=1$, hence by the recurrence hypothesis $\mathbf{y}_{1:i}$ represents a path $s \to u$ in $D^i$ and by adding $e_{i+1}$ $\mathbf{y}$ represents a path $s \to v$ in $D^{i+1}$.
                \end{itemize}
                \item if $y_{i+1}=0$, then $C_u^i(\mathbf{y}_{1:i})=C_u^{i+1}(\mathbf{y})=1$, hence by the recurrence hypothesis $\mathbf{y}_{1:i}$ represents a path $s \to v$ in $D^i$ and by not adding edge $e_{i+1}$ $\mathbf{y}$ represents a path $s \to v$ in $D^{i+1}$.
            \end{itemize}
        \end{itemize}
    \end{itemize}
\end{proof}

\begin{remark}
    It is easy to see that $C_v^i$ is even an \texttt{OBDD} as every $\lor$-nodes are decision nodes ordered by the topological order of the edges.
\end{remark}

\section{Related work} \label{sec:related}


\paragraph{Graphical models} Graphical models allow to specify a family of distributions over a finite set of variables by means of a graph \cite{Lauritzen1996}. The graph encodes a set of properties (\eg factorization, independence, etc.) shared by all distributions in the family. These properties can be exploited to produce compressed representations and efficient inference algorithms \cite{Kschischang2001,Kwisthout2011}. In the context of probabilistic reasoning, the primal graph of a \texttt{CNF} $\kappa$, which has one vertex for each variable in $\kappa$ and an edge between two variables if they appear in the same clique, specifies to which graphical model the family of exponential distributions conditioned on $\kappa$ (\ie $\{\mathcal{P}(\cdot | \mathbf{p}, \kappa)\}_{\mathbf{p} \in ]0,1[^k}$) belong. In particular, traditional algorithms for graphical models can be leveraged to solve \texttt{PQE} and \texttt{MPE} problems in time $\mathcal{O}(k2^{\tau(\kappa)})$ where $k$ is the number of variables and $\tau(\kappa)$ the tree-width of the primal graph of $\kappa$. Such algorithms were used to implement semantic conditioning in \cite{Deng2014}. These algorithms are similar to knowledge compilation in that they first compute offline another representation of the distribution (in this case a junction tree), before running inference algorithms (for instance the sum-product or max-product algorithms) on this new representation. However, they are typically less efficient than knowledge compilation techniques since they are limited to bounded tree-width instances.

\paragraph{Compilation complexity} \label{sec:cc} Compilation complexity is interested in cases of non-efficient compilation from a source language $\mathtt{L}_s$ into a target language $\mathtt{L}_t$ tractable for a given reasoning problem (\ie $\mathtt{L}_s \to_c \mathtt{L}_t$ but not $\mathtt{L}_s \to_{ec} \mathtt{L}_t$). In such cases, knowledge compilation can still remain relevant if many hard queries have to be answered on a single theory. Indeed, a \textit{hard} poly-size compilation step to a target language (that is tractable for the required type of queries) can be done \textit{offline} in exchange of \textit{online} tractable queries on the compiled theory. Compilation complexity theory \cite{Cadoli2002} is the formal study of such classes of complexity.


\section{Conclusion} \label{sec:conclusion}
This paper studies the scalability of neurosymbolic techniques based on probabilistic reasoning. After introducing a unified framework for propositional knowledge representation and probabilistic reasoning, we identify the key probabilistic reasoning problems (\ie \texttt{MPE}, \texttt{ThreshEnum}, \texttt{PQE} and \texttt{EQE}) on which several neurosymbolic techniques rely. We use knowledge compilation to \texttt{d-DNNF} to show tractability results for several succinct languages that represent popular types of knowledge in informed classification. However, we also point out the limits of this approach, in particular its inability to exploit the complexity gap between optimization/enumeration problems (\texttt{MPE}/\texttt{ThreshEnum}) and counting problems (\texttt{PQE}/\texttt{EQE}). We bring together previously known and our new results into the first complexity map for probabilistic reasoning that includes counting, optimization and enumeration problems. We hope this work will help neurosymbolic AI practitioners navigate the scalability landscape of probabilistic neurosymbolic techniques.

Future directions for our research include completing this complexity map with: other representation languages, other (probabilistic) reasoning problems (\eg marginal \texttt{MPE} and \texttt{PQE} queries), a sharper understanding of the complexity gap and a characterization of space complexity for enumeration problems. We would also like to explore compilation complexity classes for probabilistic reasoning and approximate algorithms for counting problems: in particular, we would like to identify which representation languages belong to \texttt{Comp-P} or \texttt{Comp-\#P} and which admit a fully polynomial-time approximation scheme (\texttt{FPTAS}) for \texttt{PQE} and \texttt{EQE}. Finally, a practical study of computation times for probabilistic neurosymbolic techniques would be interesting, to understand how much these theoritical results are informative regarding their domain of scalability in practice.

\section*{Acknowledgments}
This work has been supported by the French government under the "France 2030” program, as part of the SystemX Technological Research Institute within the SMD project.

\bibliography{main}

\end{document}